\documentclass[sigconf]{acmart}

\newcommand{\fires}{\emph{FIRES}}
\DeclareMathOperator*{\argmax}{arg\,max}
\newtheorem{property}{Property}

\usepackage{color,soul} % just for highlighting text

\usepackage{tikz}
\usetikzlibrary{bayesnet}
\usepackage{bm}

\usepackage{adjustbox}
\usepackage{subcaption}
\usepackage{multirow}

\usepackage[boxed, ruled]{algorithm2e}
\usepackage{enumitem}

\AtBeginDocument{%
  \providecommand\BibTeX{{%
    \normalfont B\kern-0.5em{\scshape i\kern-0.25em b}\kern-0.8em\TeX}}}

\copyrightyear{2020}
\acmYear{2020}
\setcopyright{acmcopyright}\acmConference[KDD '20]{Proceedings of the 26th ACM SIGKDD Conference on Knowledge Discovery and Data Mining}{August 23--27, 2020}{Virtual Event, CA, USA}
\acmBooktitle{Proceedings of the 26th ACM SIGKDD Conference on Knowledge Discovery and Data Mining (KDD '20), August 23--27, 2020, Virtual Event, CA, USA}
\acmPrice{15.00}
\acmDOI{10.1145/3394486.3403200}
\acmISBN{978-1-4503-7998-4/20/08}

\begin{document}

%%
%% The "title" command has an optional parameter,
%% allowing the author to define a "short title" to be used in page headers.
\title{Leveraging Model Inherent Variable Importance for Stable Online Feature Selection}

%%
%% The "author" command and its associated commands are used to define
%% the authors and their affiliations.
%% Of note is the shared affiliation of the first two authors, and the
%% "authornote" and "authornotemark" commands
%% used to denote shared contribution to the research.
\author{Johannes Haug}
\affiliation{%
  \institution{University of Tuebingen}
  \streetaddress{Sand}
  \city{Tuebingen}
  \country{Germany}
}
\email{johannes-christian.haug@uni-tuebingen.de}

\author{Martin Pawelczyk}
\affiliation{%
  \institution{University of Tuebingen}
  \streetaddress{Sand}
  \city{Tuebingen}
  \country{Germany}
}
\email{martin.pawelczyk@uni-tuebingen.de}

\author{Klaus Broelemann}
\affiliation{%
  \institution{Schufa Holding AG}
  \streetaddress{Sand}
  \city{Wiesbaden}
  \country{Germany}
}
\email{klaus.broelemann@schufa.de}

\author{Gjergji Kasneci}
\affiliation{%
  \institution{University of Tuebingen}
  \streetaddress{Sand}
  \city{Tuebingen}
  \country{Germany}
}
\email{gjergji.kasneci@uni-tuebingen.de}

%%
%% By default, the full list of authors will be used in the page
%% headers. Often, this list is too long, and will overlap
%% other information printed in the page headers. This command allows
%% the author to define a more concise list
%% of authors' names for this purpose.
%%\renewcommand{\shortauthors}{Haug, et al.}
\fancyhead{} % remove headers

%%
%% The abstract is a short summary of the work to be presented in the
%% article.
\begin{abstract}
Feature selection can be a crucial factor in obtaining robust and accurate predictions. Online feature selection models, however, operate under considerable restrictions; they need to efficiently extract salient input features based on a bounded set of observations, while enabling robust and accurate predictions. In this work, we introduce \fires, a novel framework for online feature selection. The proposed feature weighting mechanism leverages the importance information inherent in the parameters of a predictive model. By treating model parameters as random variables, we can penalize features with high uncertainty and thus generate more stable feature sets. Our framework is generic in that it leaves the choice of the underlying model to the user. Strikingly, experiments suggest that the model complexity has only a minor effect on the discriminative power and stability of the selected feature sets. In fact, using a simple linear model, \fires\ obtains feature sets that compete with state-of-the-art methods, while dramatically reducing computation time. In addition, experiments show that the proposed framework is clearly superior in terms of feature selection stability.
\end{abstract}

%%
%% The code below is generated by the tool at http://dl.acm.org/ccs.cfm.
%% Please copy and paste the code instead of the example below.
%%


%%
%% Keywords. The author(s) should pick words that accurately describe
%% the work being presented. Separate the keywords with commas.

%% A "teaser" image appears between the author and affiliation
%% information and the body of the document, and typically spans the
%% page.
% \begin{teaserfigure}
%   \includegraphics[width=\textwidth]{sampleteaser}
%   \caption{Seattle Mariners at Spring Training, 2010.}
%   \Description{Enjoying the baseball game from the third-base
%   seats. Ichiro Suzuki preparing to bat.}
%   \label{fig:teaser}
% \end{teaserfigure}

%%
%% This command processes the author and affiliation and title
%% information and builds the first part of the formatted document.
\maketitle

%--- INTRODUCTION -----------------%
\section{Introduction}
Online feature selection has been shown to improve the predictive quality in high-dimensional streaming applications. Aiming for real time predictions, we need online feature selection models that are both effective and efficient. Recently, we also witness a demand for interpretable and stable machine learning methods \cite{rudin2019stop}. Yet, the stability of feature selection models remains largely unexplored.

In practice, feature selection is primarily used to mitigate the so-called \emph{curse of dimensionality}. This term refers to the negative effects on the predictive model that we often observe in high-dimensional applications; such as weak generalization abilities, for example. In this context, feature selection has successfully been applied to both offline and online machine learning applications \cite{guyon2003introduction,chandrashekar2014survey,bolon2015recent,li2018feature}.

Data streams are a potentially unbounded sequence of time steps. As such, data streams preclude us from storing all observations that appear over time. Consequently, at each time step $t$, feature selection models can analyse only a subset of the data to identify relevant features. Besides, temporal dynamics, e.g. concept drift, may change the underlying data distributions and thereby shift the attentive relation of features \cite{gama2014survey}. To sustain high predictive power, online feature selection models must be flexible with respect to shifting distributions. For this reason, online feature selection usually proves to be more challenging than batch feature selection.

Online models should not only be flexible with regard to the data distribution, but also robust against small variations of the input or random noise. Otherwise, the reliability of a model may suffer. Robustness is also one of the key requirements of a report published by the European Commission \cite{hamon2020robustness}. For online feature selection, this means that we aim to avoid drastic variations of the selected features in subsequent time steps. Yet, whenever a data distribution changes, we must adjust the feature set accordingly. Only few authors have examined the stability of feature selection models \cite{kalousis2005stability,nogueira2017stability,barddal2019boosting}, which leaves plenty of room for further investigation.

Ideally, we aim to uncover a stable set of discriminative features at every time step $t$. Feature selection stability, e.g. defined by \cite{nogueira2017stability}, usually corresponds to a low variation of the selected feature set. We could reduce the variation and thereby maximize the stability of a feature set by selecting only those features we are certain about. Still, we aim to select a feature set that is highly discriminative with respect to the current data generating distribution. In order to meet both requirements, one would have to weigh the features according to their importance and uncertainty regarding the decision task at hand. We translate these considerations into three sensible properties for stable feature weighting in data streams:
\begin{property}
\label{prop:attention}
    \emph{(Attentive Weights)} Feature weights must preserve attentive relations of features. Given an arbitrary feature $x_j$, let $\mu_{tj}$ be its measured importance and let $\omega_{tj}$ be its weight at some time step $t$. The feature weight $\omega_{tj}$ must be a function of $\mu_{tj}$, such that $\mu_{tj} = 0 \Rightarrow \omega_{tj} \leq 0$.
\end{property}

Intuitively, we would expect the weight of a feature to be exactly zero, if its associated importance is zero. But since there might be dependencies between the different weights, we allow the weights in such cases to become smaller than zero, thus ensuring a higher flexibility in the possible weight configurations.

\begin{property}
\label{prop:monotonicity}
    \emph{(Monotonic Weights)} Feature weights must be a strictly monotonic function of importance and uncertainty. Given two arbitrary features $x_i \neq x_j$, let $|\mu_{ti}|, |\mu_{tj}|$ be their absolute measured importance and let $\sigma_{ti}, \sigma_{tj} \geq 0$ be the respective measure of uncertainty at some time step $t$. Two conditions must hold:
    \begin{itemize}
        \item [2.1] Given $|\mu_{ti}| = |\mu_{tj}|$, the following holds: $\sigma_{ti} \geq \sigma_{tj} \Leftrightarrow \omega_{ti} \leq \omega_{tj}$. Otherwise, if $\sigma_{ti} < \sigma_{tj}$, meaning we are more certain about feature $x_i$'s than feature $x_j$'s discriminative power, it holds that $\omega_{ti} > \omega_{tj}$, and vice versa.
        \item [2.2] Given $\sigma_{ti} = \sigma_{tj}$, the following holds: $|\mu_{ti}| \geq |\mu_{tj}| \Leftrightarrow \omega_{ti} \geq \omega_{tj}$. Otherwise, if $|\mu_{ti}| < |\mu_{tj}|$, meaning that feature $x_i$ is less discriminative than feature $x_j$, it holds that $\omega_{ti} < \omega_{tj}$, and vice versa.
    \end{itemize}
\end{property}

The second property specifies that features with high importance and low uncertainty must be given a higher weight than features with low importance and high uncertainty.

\begin{property}
\label{prop:consistency}
    \emph{(Consistent Weights)} For a stable target distribution, feature weights must eventually yield a consistent ranking. Let $\mathcal{R}(\omega_t)$ be the ranking of features according to their weights at time step $t$. Assume $\exists \bar{t}$, such that $P(y_t|x_t) = P(y_{t+1}|x_{t+1})$ $\forall t \geq \bar{t}$. As $t \geq \bar{t} \rightarrow \infty$, it holds that $\mathcal{R}(\omega_t) = \mathcal{R}(\omega_{t-1})$.
\end{property}

Consistent weights eventually yield a stable ranking of features, if the conditional target distribution does not change anymore.

%----------- GRAPHICAL MODEL -----------------%
\begin{figure}[t]
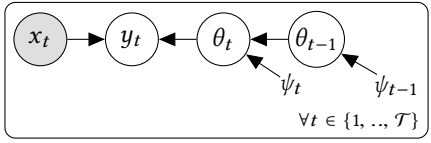

  \centering
  \tikz{
    \node[latent] (y) {$y_t$};
    \node[obs, left=0.5cm of y] (x) {$x_t$};
    \node[latent, right=0.5cm of y] (theta) {$\theta_t$};
    \node[const, below right= 0.2cm and 0.5cm of theta] (psi) {$\psi_t$};
    \node[latent, right=0.5cm of theta] (theta2) {$\theta_{t-1}$};
    \node[const, below right= 0.2cm and 0.5cm of theta2] (psi2) {$\psi_{t-1}$};
    \edge {x} {y};
    \edge {theta} {y};
    \edge {psi} {theta};
    \edge {theta2} {theta};
    \edge {psi2} {theta2};
    \plate {plate} {(y)(x)(theta)(psi)(theta2)(psi2)} {$\forall t \in \{1,..,\mathcal{T}\}$};
  }
  \caption{\emph{Graphical Model}. The target variable $y_t$  at time step $t \in \{1,..,\mathcal{T}\}$ depends on a feature vector $x_t$ (observed variable = shaded grey) and model parameters $\theta_t$. We treat the parameters $\theta_t$ as random variables that are parameterised by $\psi_t$, which in turn contains information about feature importance and uncertainty. We update the distribution of $\theta$ at every time step $t$ with respect to the new observations.}
  \label{fig:graphical_model}
\end{figure}
%----------- GRAPHICAL MODEL -----------------%

These three properties can help guide the development of robust feature weighting schemes. To the best of our knowledge, they are also the first formal definition of valuable properties for stable feature weighting in data streams.

In order to fulfill the properties specified before, we need a measure of feature importance and corresponding uncertainty. If a predictive model $\mathcal{M}_\theta$ is trained on the input features, we expect its parameters $\theta$ to contain the required information. Specifically, we can extract the latent importance and uncertainty of features regarding the prediction at time step $t$, by treating $\theta_t$ as a random variable. Accordingly, $\theta_t$ is parameterised by $\psi_t$, which contains the sufficient statistics. Given that all parameters initially follow the same distribution, we can optimize $\psi_t$ for every new observation using gradient updates (e.g. stochastic gradient ascent). If we update $\psi_t$ at every time step, the parameters contain the most current information about the importance and uncertainty of input features. These considerations translate into the graphical model in Figure \ref{fig:graphical_model} and form the basis of a novel framework for \emph{\underline{F}ast, \underline{I}nterpretable and \underline{R}obust feature \underline{E}valuation and \underline{S}election} (\fires). \fires\ selects features with high importance, penalizing high uncertainty, to generate a feature set that is both discriminative and stable.

In summary, the contributions of this work are:
\begin{enumerate}[label=(\alph*)]
    \item A specification of sensible properties which, when fulfilled, help to create more reliable feature weights in data streams.
    \item A flexible and generic framework for online feature weighting and selection that satisfies the proposed properties.
    \item A concrete application of the proposed framework to three common model families: Generalized Linear Models (GLM) \cite{nelder1972generalized}, Artificial Neural Nets (ANN) and Soft Decision Trees (SDT) \cite{frosst2017distilling,irsoy2012soft} (an open source implementation can be found at \url{https://github.com/haugjo/fires}).
    \item An evaluation on several synthetic and real-world data sets, which shows that the proposed framework is superior to existing work in terms of speed, robustness and predictive accuracy.
\end{enumerate}

The remainder of this paper is organized as follows: We introduce the general objective and feature weighting scheme of our framework in Section \ref{sec:fires_framework}. Here, we also show that \fires~ produces attentive, monotonic and consistent weights as defined by the Properties \ref{prop:attention} to \ref{prop:consistency}. We describe three explicit specifications of \fires~in Section \ref{sec:fires_mechanics}. In Section \ref{sec:stability}, we show how feature selection stability can be evaluated in streaming applications. Finally, we cover related work in Section \ref{sec:related_work} and evaluate our framework in a series of experiments in Section \ref{sec:experiments}.

%-----------NOTATION TABLE---------%
\begin{table}[t]
\caption{Important Variables and Notation}
\label{tab:notation}
\centering
    \begin{adjustbox}{max width=\columnwidth}
        \begin{tabular}{cll}
        \toprule
        \textbf{Notation} & \textbf{Description}\\ 
        \cmidrule(lr){1-1} \cmidrule(lr){2-3}
        $t \in \{1,..,\mathcal{T}\}$ & \multicolumn{2}{p{5cm}}{Time step.} \\
        $x_t \in \mathbb{R}^{B \times J}$ & \multicolumn{2}{p{5cm}}{Observations at time step $t$ with $J$ features and a batch size of $B$.}\\
        $y_t = [y_{t1},..,y_{tB}]$ & \multicolumn{2}{p{5cm}}{Target variable at time step $t$.}\\
        $\theta_t = [\theta_{t1},..,\theta_{tk},..,\theta_{tK}]$ & \multicolumn{2}{p{5cm}}{Parameters of a model $\mathcal{M}_{\theta_t}$; $K \geq J$.}\\
        $P(\theta_{tk}|\psi_{tk})$ & \multicolumn{2}{p{5cm}}{Probability distribution of $\theta_{tk}$, parameterised by $\psi_{tk}$.}\\
        $\omega_t = [\omega_{t1},..,\omega_{tj},..,\omega_{tJ}]$ & \multicolumn{2}{p{5cm}}{Feature weights at time step $t$.}\\
        $M \in \mathbb{N}$ & \multicolumn{2}{p{5cm}}{Number of selected features; $M \leq J$.}\\
        \bottomrule
        \end{tabular}
    \end{adjustbox}
\end{table}
%-----------NOTATION TABLE---------%

%-----------------------------FIRES TIKZ MODEL -----------------------------------
\begin{figure*}[ht]
    \begin{tikzpicture}[x=1pt,y=1pt,yscale=-0.6,xscale=0.6]
    %uncomment if require: \path (0,180); %set diagram left start at 0, and has height of 180

    %Curve Lines [id:da4494002640083159] 
    \draw [fill={rgb, 255:red, 155; green, 155; blue, 155 }  ,fill opacity=0.4 ]   (292.5,39.04) .. controls (314.61,41.52) and (311.56,13.69) .. (323.38,14) .. controls (335.2,14.31) and (330.63,41.82) .. (353.5,39.35) ;
    %Straight Lines [id:da6995962955472026] 
    \draw [fill={rgb, 255:red, 155; green, 155; blue, 155 }  ,fill opacity=0.4 ] [dash pattern={on 4.5pt off 4.5pt}]  (323.38,14) -- (323.38,39.04) ;
    
    %Straight Lines [id:da4383133307577769] 
    \draw    (72.5,80) -- (72.5,44) ;
    \draw [shift={(72.5,41)}, rotate = 450] [fill={rgb, 255:red, 0; green, 0; blue, 0 }  ][line width=0.08]  [draw opacity=0] (8.93,-4.29) -- (0,0) -- (8.93,4.29) -- cycle    ;
    %Straight Lines [id:da39271231057054723] 
    \draw    (72.5,150) -- (72.5,112) ;
    \draw [shift={(72.5,109)}, rotate = 450] [fill={rgb, 255:red, 0; green, 0; blue, 0 }  ][line width=0.08]  [draw opacity=0] (8.93,-4.29) -- (0,0) -- (8.93,4.29) -- cycle    ;
    %Straight Lines [id:da4018164213648314] 
    \draw    (163,10.5) -- (153,10.5) -- (153,171) -- (163,171) ;
    %Curve Lines [id:da5511589717598009] 
    \draw [fill={rgb, 255:red, 155; green, 155; blue, 155 }  ,fill opacity=0.4 ]   (292.5,100.04) .. controls (318.5,102) and (315.87,59.95) .. (322.75,59.95) .. controls (329.63,59.95) and (326.5,102) .. (353.5,100.35) ;
    %Straight Lines [id:da1083577799009976] 
    \draw  [dash pattern={on 4.5pt off 4.5pt}]  (323.38,59.95) -- (323.38,100.04) ;
    %Curve Lines [id:da3133761022689745] 
    \draw [fill={rgb, 255:red, 155; green, 155; blue, 155 }  ,fill opacity=0.4 ]   (292.5,161.04) .. controls (312.25,162.45) and (305.87,140.7) .. (323.25,140.95) .. controls (340.63,141.2) and (333.75,161.95) .. (353.5,161.35) ;
    %Straight Lines [id:da8088553988560201] 
    \draw [fill={rgb, 255:red, 155; green, 155; blue, 155 }  ,fill opacity=0.4 ] [dash pattern={on 4.5pt off 4.5pt}]  (323.38,140.95) -- (323.38,161.04) ;
    
    %Straight Lines [id:da05406551753126698] 
    \draw  [dash pattern={on 0.84pt off 2.51pt}]  (633,58.5) -- (633,79.5) ;
    %Straight Lines [id:da45977389853088146] 
    \draw  [dash pattern={on 0.84pt off 2.51pt}]  (633,105.5) -- (633,125.5) ;
    %Straight Lines [id:da8776906248329777] 
    \draw    (622,30.25) -- (613,30.25) -- (613,150.25) -- (623,150.25) ;
    %Straight Lines [id:da10873107623833822] 
    \draw    (642.5,150.02) -- (652.5,150.02) -- (652.5,31) -- (643.5,31) ;
    %Straight Lines [id:da507441863147835] 
    \draw  [dash pattern={on 0.84pt off 2.51pt}]  (252.5,51) -- (252.5,70) ;
    %Straight Lines [id:da507479561420229] 
    \draw  [dash pattern={on 0.84pt off 2.51pt}]  (252.5,111) -- (252.5,130) ;
    %Straight Lines [id:da7719689335542208] 
    \draw  [dash pattern={on 4.5pt off 4.5pt}]  (102.5,80) -- (153,10.5) ;
    %Straight Lines [id:da3115286881613146] 
    \draw  [dash pattern={on 4.5pt off 4.5pt}]  (102.5,109) -- (153,171) ;
    %Shape: Rectangle [id:dp5761671183289292] 
    \draw  [fill={rgb, 255:red, 155; green, 155; blue, 155 }  ,fill opacity=0.4 ] (42.5,80) -- (102.5,80) -- (102.5,109) -- (42.5,109) -- cycle ;
    %Straight Lines [id:da6848429981298005] 
    \draw    (352.72,170.3) -- (362.72,170.34) -- (363.39,9.85) -- (353.39,9.8) ;
    %Straight Lines [id:da613701180187417] 
    \draw    (363,89.75) -- (390,90.2) ;
    \draw [shift={(393,90.25)}, rotate = 180.95] [fill={rgb, 255:red, 0; green, 0; blue, 0 }  ][line width=0.08]  [draw opacity=0] (8.93,-4.29) -- (0,0) -- (8.93,4.29) -- cycle    ;
    %Straight Lines [id:da27688162991390386] 
    \draw    (403,9.75) -- (393,9.75) -- (393,90.25) -- (393,170.25) -- (403,170.25) ;
    %Shape: Boxed Line [id:dp5668559423195831] 
    \draw    (463.72,170.8) -- (473.72,170.84) -- (474.39,10.35) -- (464.39,10.3) ;
    %Straight Lines [id:da07059392609304038] 
    \draw  [dash pattern={on 0.84pt off 2.51pt}]  (432.75,50.5) -- (432.75,70) ;
    %Straight Lines [id:da050164408910600145] 
    \draw  [dash pattern={on 0.84pt off 2.51pt}]  (433.25,110.5) -- (433.25,130) ;
    %Straight Lines [id:da7182758642102316] 
    \draw    (474.17,90) -- (500,90.22) ;
    \draw [shift={(503,90.25)}, rotate = 180.5] [fill={rgb, 255:red, 0; green, 0; blue, 0 }  ][line width=0.08]  [draw opacity=0] (8.93,-4.29) -- (0,0) -- (8.93,4.29) -- cycle    ;
    %Straight Lines [id:da4498193003536841] 
    \draw    (513,9.75) -- (503,9.75) -- (503,90.25) -- (503,170.25) -- (513,170.25) ;
    %Shape: Boxed Line [id:dp92342617019765] 
    \draw    (562.38,170.3) -- (572.38,170.34) -- (573.05,9.85) -- (563.05,9.8) ;
    %Straight Lines [id:da06422585956834759] 
    \draw  [dash pattern={on 0.84pt off 2.51pt}]  (536.83,50.17) -- (536.83,70.17) ;
    %Straight Lines [id:da07350002916820553] 
    \draw  [dash pattern={on 0.84pt off 2.51pt}]  (536.83,111.33) -- (536.83,131.33) ;
    %Straight Lines [id:da5507207579757287] 
    \draw  [dash pattern={on 4.5pt off 4.5pt}]  (573.05,9.85) -- (613,30.25) ;
    %Straight Lines [id:da4590760374751577] 
    \draw  [dash pattern={on 4.5pt off 4.5pt}]  (572.38,170.34) -- (613,150.25) ;
    
    % Text Node
    \draw (73.5,30.6) node [anchor=south] [inner sep=0.75pt]    {$P( y|x,\theta )$};
    % Text Node
    \draw (72.5,153.4) node [anchor=north] [inner sep=0.75pt]    {$[ y,x_{1} ,...,x_{j} ,...,x_{J}]$};
    % Text Node
    \draw (225,23.4) node [anchor=north] [inner sep=0.75pt]    {$\theta _{1} \sim \mathcal{N}( \mu _{1} ,\sigma _{1}) =$};
    % Text Node
    \draw (224,83.4) node [anchor=north] [inner sep=0.75pt]    {$\theta _{k} \sim \mathcal{N}( \mu _{k} ,\sigma _{k}) =$};
    % Text Node
    \draw (225,143.4) node [anchor=north] [inner sep=0.75pt]    {$\theta _{K} \sim \mathcal{N}( \mu _{K} ,\sigma _{K}) =$};
    % Text Node
    \draw (74,92.5) node    {$\mathcal{M}_{\mathbf{\theta }}$};
    % Text Node
    \draw (633.5,80.85) node [anchor=north] [inner sep=0.75pt]    {$\omega _{j}$};
    % Text Node
    \draw (633,55.1) node [anchor=south] [inner sep=0.75pt]    {$\omega _{1}$};
    % Text Node
    \draw (633,122.9) node [anchor=north] [inner sep=0.75pt]    {$\omega _{J}$};
    % Text Node
    \draw (434,48.6) node [anchor=south] [inner sep=0.75pt]    {$\left(\frac{\partial P}{\partial \mu _{1}} ,\frac{\partial P}{\partial \sigma _{1}}\right)$};
    % Text Node
    \draw (537.83,43.77) node [anchor=south] [inner sep=0.75pt]    {$( \mu '_{1} ,\sigma '_{1})$};
    % Text Node
    \draw (434,108.1) node [anchor=south] [inner sep=0.75pt]    {$\left(\frac{\partial P}{\partial \mu _{k}} ,\frac{\partial P}{\partial \sigma _{k}}\right)$};
    % Text Node
    \draw (434,168.6) node [anchor=south] [inner sep=0.75pt]    {$\left(\frac{\partial P}{\partial \mu _{K}} ,\frac{\partial P}{\partial \sigma _{K}}\right)$};
    % Text Node
    \draw (537.83,103.77) node [anchor=south] [inner sep=0.75pt]    {$( \mu '_{k} ,\sigma '_{k})$};
    % Text Node
    \draw (537.83,163.77) node [anchor=south] [inner sep=0.75pt]    {$( \mu '_{K} ,\sigma '_{K})$};

    \end{tikzpicture}
    \caption{\emph{The FIRES Framework.} By treating the parameters of a model $\mathcal{M}_\theta$ as random variables, the proposed framework is able to extract the importance and uncertainty of every input feature with respect to the prediction. For illustration, let $\theta_k~\forall k$ be normally distributed parameters. \fires\ optimizes the mean $\mu_k$ (importance) and standard deviation $\sigma_k$ (uncertainty) of all $K$ parameters, by using gradient updates. Based on the updated parameters, \fires\ then computes feature weights $\omega_j~\forall j$.}
    \label{fig:fires_mechanics}
\end{figure*}
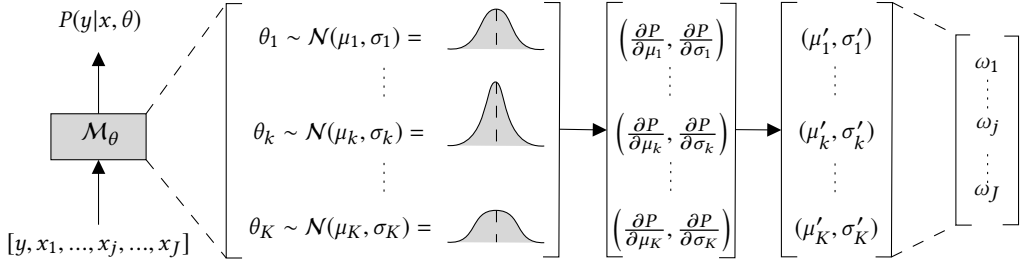
%-----------------------------FIRES TIKZ MODEL---------------------------------------

%--- FIRES FRAMEWORK -----------------%
\section{The FIRES Framework}
\label{sec:fires_framework}
The parameters of a predictive model encode every input feature's importance in the prediction. By treating model parameters as random variables, we can quantify the importance and uncertainty of each feature. These estimates can then be used for feature weighting at every time step. This is the core idea of the proposed framework \fires, which is illustrated in Figure \ref{fig:fires_mechanics}. Table \ref{tab:notation} introduces relevant notation. Let $\theta$ be a vector of model parameters whose distribution is parameterised by $\psi$. Specifically, for every parameter $\theta_k$ we choose a distribution, so that $\psi_k$ comprises an importance and uncertainty measure regarding the predictive power of $\theta_k$. We then look for the $K$ distribution parameters that optimize the prediction.  %For example, we may choose a fully factorized Gaussian distribution, where each factor is governed by a mean and variance.
\\\\
\textbf{\emph{General Objective:}} We translate these considerations into an objective: Find the distribution parameters $\psi$ that maximize the log-likelihood given the observed data, i.e.
\begin{equation}
\label{eq:argmax_likelihood}
    \argmax_{\Psi_\mathcal{T}} \mathcal{L}(\Psi_\mathcal{T}, Y_\mathcal{T}, X_\mathcal{T}) = \argmax_{\Psi_\mathcal{T}} \sum^\mathcal{T}_{t=1} log\, P(y_t|x_t,\psi_t),
\end{equation}
with observations $x_t$ and corresponding labels $y_t$. Note that we optimize the logarithm of the likelihood, because it is easier to compute. By the nature of data streams, we never have access to the full data set before time step $\mathcal{T}$. Hence, we cannot compute \eqref{eq:argmax_likelihood} in closed form. Instead, we optimize $\psi$ incrementally using stochastic gradient ascent. Alternatively, one could also use online variational Bayes \cite{broderick2013streaming} to infer posterior parameters. However, gradient based optimization is very efficient, which can be a considerable advantage in data stream applications. The gradient of the log-likelihood with respect to $\psi_t$ is
\begin{equation}
\label{eq:gradient_likelihood}
    \nabla_{\psi_t} \mathcal{L} = \frac{1}{P(y_t|x_t,\psi_t)} \nabla_{\psi_t} P(y_t|x_t,\psi_t),
\end{equation}
with the marginal likelihood
\begin{align}
\label{eq:gradient_marginal}
    P(y_t|x_t,\psi_t) = \int P(y_t|x_t,\theta_t) \: P(\theta_t| \psi_t) \: d\theta_t.
\end{align}
We update $\psi_t$ with a learning rate $\alpha$ in iterations of the form:
\begin{equation}
\label{eq:gradient_ascent}
    \psi_t' = \psi_t + \alpha \nabla_{\psi_t} \mathcal{L}
\end{equation}
\\
\textbf{\emph{Feature Weighting Scheme:}} Given the updated distribution parameters, we can compute feature weights in a next step. Note that we may have a one-to-many mapping between input features and model parameters, depending on the predictive model at hand. In this case, we have to aggregate relevant parameters, which we will show in Section \ref{sec:aggregation}. In the following, we assume that there is a single (aggregated) parameter per input feature. Let $\mu_t, \sigma_t$ be the estimated importance and uncertainty of features at time step $t$. Our goal is to maximize the feature weights $\omega_t$ whenever a feature is of high importance and to minimize the weights under high uncertainty. In this way, we aim to obtain optimal feature weights that are both discriminative and stable. We express this trade-off in an objective function:
\begin{align}
\label{eq:weight_objective}
    &\argmax_{\omega_\mathcal{T}} \sum^\mathcal{T}_{t=1} \bigg(\underbrace{\sum^J_{j=1} \omega_{tj} \mu^2_{tj}}_{\text{importance}} - \underbrace{\lambda_s \sum^J_{j=1} \omega_{tj} \sigma^2_{tj}}_{\text{uncertainty}} - \underbrace{\lambda_r \sum^J_{j=1} \omega_{tj}^2}_{\text{regularizer}} \bigg) \nonumber\\
    = &\argmax_{\omega_\mathcal{T}} \sum^\mathcal{T}_{t=1} \sum^J_{j=1} \omega_{tj} (\mu_{tj}^2 - \lambda_s \sigma^2_{tj} - \lambda_r \omega_{tj})
\end{align}
Note that we regularize the objective with the squared $\ell2$-norm to obtain small feature weights. Besides, we specify two scaling factors $\lambda_s \geq 0$ and $\lambda_r \geq 0$, which scale the uncertainty penalty and regularization term, respectively. These scaling factors allow us to adjust the sensitivity of the weighting scheme with respect to both penalties. For example, if we have a critical application that requires high robustness (e.g. in medicine), we can increase $\lambda_s$ to impose a stronger penalty on uncertain parameters. Choosing an adequate $\lambda_s$ is usually not trivial. In general, a larger $\lambda_s$ improves the robustness of the feature weights, but limits the flexibility of the model in the face of concept drift.

From now on, we omit time indices to avoid overloading the exposition, e.g. $\omega_t = \omega$. The considerations that follow account for a single time step $t$. We further assume a batch size of $B=1$. To maximize \eqref{eq:weight_objective} for some $\omega_j$, we evaluate the partial derivative at zero:
\begin{align}
\label{eq:omega}
    \frac{\partial}{\partial \omega_j} &= \mu^2_j - \lambda_s \sigma^2_j - 2\lambda_r\omega_j \overset{!}{=} 0 \nonumber\\
    \Leftrightarrow~ -2\lambda_r\omega_j &= -\mu^2_j + \lambda_s \sigma^2_j \nonumber\\
    \Leftrightarrow~ \omega^*_j &= \frac{1}{2\lambda_r} \left( \mu^2_j - \lambda_s \sigma^2_j\right)
\end{align}
%Note that \eqref{eq:omega} computes weights as a function of predictive importance (denoted by $\mu$) and uncertainty (denoted by $\sigma$). 
In accordance with Property \ref{prop:attention} to \ref{prop:consistency}, we show that the weights obtained from \eqref{eq:omega} are attentive, monotonic and consistent:
\begin{lemma}
    Equation \eqref{eq:omega} produces attentive weights as specified by Property \ref{prop:attention}.
\end{lemma}
\begin{proof}
    This property follows immediately from \eqref{eq:omega}. Since $\sigma^2_j \geq 0$ and $\lambda_s, \lambda_r \geq 0$, for $\mu_j = 0$ we get $\omega_j \leq 0$.
\end{proof}
\begin{lemma}
    Equation \eqref{eq:omega} produces monotonic weights as specified by Property \ref{prop:monotonicity}.
\end{lemma}
\begin{proof}
    Given two features $x_i \neq x_j$, Property \ref{prop:monotonicity} specifies two sub-criteria, which we proof independently:
    \begin{itemize}
        \item [$\boldsymbol{>}$] Given $|\mu_i| = |\mu_j|$, we can show $\sigma_i \geq \sigma_j \Leftrightarrow \omega_i \leq \omega_j$:
        \begin{align*}
            \sigma_i &\geq \sigma_j\\
            \Leftrightarrow c - b\sigma_i^2 &\leq c - b\sigma_j^2; \:\:\:  b,c \geq 0
        \end{align*}
        For $c=\frac{1}{2\lambda_r}\mu^2_i=\frac{1}{2\lambda_r}\mu^2_j$ and $b=\frac{\lambda_s}{2\lambda_r}$ we get
        \begin{align*}
            \Leftrightarrow \frac{1}{2\lambda_r}\mu^2_i - \frac{\lambda_s}{2\lambda_r} \sigma^2_i &\leq \frac{1}{2\lambda_r}\mu^2_j - \frac{\lambda_s}{2\lambda_r} \sigma^2_j\\
            \Leftrightarrow \omega_i &\leq \omega_j
        \end{align*}
        
        \item [$\boldsymbol{>}$] Given $\sigma_i = \sigma_j$, we can show $|\mu_i| \geq |\mu_j| \Leftrightarrow \omega_i \geq \omega_j$:
        \begin{align*}
            |\mu_i| &\geq |\mu_j|\\
            \Leftrightarrow b \mu_i^2 - c &\geq b \mu_j^2 - c; \:\:\:  b,c \geq 0
        \end{align*}
        for $b=\frac{1}{2\lambda_r}$ and $c = \frac{\lambda_s}{2\lambda_r} \sigma^2_i = \frac{\lambda_s}{2\lambda_r} \sigma^2_j$, we get
        \begin{align*}
            \Leftrightarrow \frac{1}{2\lambda_r}\mu^2_i - \frac{\lambda_s}{2\lambda_r} \sigma^2_i &\geq \frac{1}{2\lambda_r}\mu^2_j - \frac{\lambda_s}{2\lambda_r} \sigma^2_j\\
            \Leftrightarrow \omega_i &\geq \omega_j
        \end{align*}
    \end{itemize}
\end{proof}
\begin{lemma}
\label{lemma:consistency}
    Equation \eqref{eq:omega} produces consistent weights as specified by Property \ref{prop:consistency}.
\end{lemma}
\begin{proof}
    $\exists \bar{t}$, such that
    \begin{equation*}
        P(y_t|x_t, \psi_t) = P(y_{t+1}|x_{t+1},\psi_{t+1}),~\forall t \geq \bar{t},
    \end{equation*}
    which by \eqref{eq:gradient_marginal} can be formulated in terms of the marginal likelihood. Consequently, since the marginal likelihood function does not change after time step $\bar{t}$, the SGA updates in \eqref{eq:gradient_ascent} will eventually converge to a local optimum. Let $t^* \geq \bar{t}$ be the time of convergence. Notably, $t^*$ specifies the time step at which $P(\theta|\psi)$ and the distribution parameters $\psi$ have been learnt, such that $\psi_t = \psi_{t^*}~\forall t \geq t^*$. By \eqref{eq:omega}, we compute feature weights $\omega$ as a function of $\psi$. Consequently, it also holds that $\omega_t = \omega_{t^*}~\forall t \geq t^*$. For the ranking of features, denoted by $\mathcal{R}(\omega)$, this implies $\mathcal{R}(\omega_t) = \mathcal{R}(\omega_{t^*})~\forall t \geq t^*$.
\end{proof}

%------------------FIRES MECHANICS ----------------------%
\subsection{Illustrating the \fires~Mechanics}
\label{sec:fires_mechanics}
The proposed framework has three variable components, which we can specify according to the requirements of the learning task at hand:
\begin{enumerate}
    \item The prior distribution of model parameters $\theta$
    \item The prior distribution of the target variable $y$
    \item The predictive model $\mathcal{M}_{\theta}$ used to compute the marginal likelihood \eqref{eq:gradient_marginal}
\end{enumerate}
The flexibility of \fires\ allows us to obtain robust and discriminative feature sets in any streaming scenario. By way of illustration, we make the following assumptions:

The prior distribution of $\theta$ must be specified so that $\psi$ contains a measure of importance and uncertainty. The Gaussian normal distribution meets this requirement. In our case, the mean value refers to the expected importance of $\theta$ in the prediction. In addition, the standard deviation measures the uncertainty regarding the expected importance. Since the normal distribution is well-explored and occurs in many natural phenomena, it is an obvious choice. Accordingly, we get $\theta_k \sim \mathcal{N}(\psi_k)~\forall k$, where $\psi_k$ comprises the mean $\mu_k$ and the standard deviation $\sigma_k$.

In general, we infer the distribution of the target variable $y$ from the data. For illustration, we assume a Bernoulli distributed target, i.e. $y \in \{-1,1\}$. Most existing work supports binary classification. The Bernoulli distribution is therefore an appropriate choice for the evaluation of our framework.

Finally, we need to choose a predictive model $\mathcal{M}_\theta$. \fires\ supports any predictive model type, as long as its parameters represent the importance and uncertainty of the input features. To illustrate this, we apply \fires\ to three common model families. %Generalized Linear Models (GLM) \cite{nelder1972generalized}, Artificial Neural Nets (ANN) and  Soft Decision Trees (SDT) \cite{frosst2017distilling,irsoy2012soft}.

\subsubsection{\textbf{FIRES And Generalized Linear Models}}
\label{sec:fires_probit}
Generalized Linear Models (GLM) \cite{nelder1972generalized} use a link function to map linear models of the form $\sum^J_{j=1} \theta_j x_j + \theta_{J+1}$ to a target distribution. Since we map to the Bernoulli space, we use the cumulative distribution function of the standard normal distribution, $\Phi$, which is known as a Probit link. Conveniently, we can associate each input feature with a single model parameter. Hence, by using a GLM, we can avoid the previously discussed parameter aggregation step. We discard $\theta_{J+1}$, as it is not linked to any specific input feature. With Lemma \ref{app:lemma_1} and \ref{app:lemma_2} (see Appendix) the marginal likelihood becomes
\begin{align*}
   P(y=1|x,\psi) &= \int \Phi\left(\sum^J_{j=1} \theta_j x_j\right)~ P(\theta| \psi)~ d\theta\\ 
   &= \Phi \left(\frac{1}{\rho} \sum^J_{j=1} \mu_j x_j\right); \:\:\: \rho = \sqrt{1 + \sum^J_{j=1} \sigma^2_j x^2_j}.
\end{align*}
Since $\Phi$ is symmetric, we can further generalize to 
\begin{equation*}
     P(y|x,\psi) = \Phi \left(\frac{y}{\rho} \sum^J_{j=1} \mu_j x_j\right).
\end{equation*}
We then compute the corresponding partial derivatives as
\begin{align*}
    \frac{\partial}{\partial\mu_j} P(y|x,\psi) &= \phi\left(\frac{y}{\rho} \sum^J_{i=1} \mu_i x_i\right) \cdot \frac{y}{\rho} x_j,\\
    \frac{\partial}{\partial\sigma_j}  P(y|x,\psi) &= \phi\left(\frac{y}{\rho} \sum^J_{i=1} \mu_i x_i\right) \cdot \frac{y}{-2\rho^3} 2x^2_j \sigma_j \sum^J_{i=1} \mu_i x_i,
\end{align*}
where $\phi$ is the probability density function of the standard normal distribution.

\subsubsection{\textbf{FIRES And Artificial Neural Nets}}
\label{sec:fires_ann}
In general, Artificial Neural Nets (ANN) make predictions through a series of linear transformations and nonlinear activations. Due to the nonlinearity and complexity of ANNs, we usually cannot solve the integral of \eqref{eq:gradient_marginal} in closed form. Instead, we approximate the marginal likelihood using the well-known Monte Carlo method. Let $f_\theta(x)$ be an ANN of arbitrary depth. We approximate \eqref{eq:gradient_marginal} by sampling $L$-times with Monte Carlo:
\begin{align*}
    P(y|x,\psi) &= \int f_{\theta}(x) \: P(\theta| \psi) \: d\theta\\
    &\approx \frac{1}{L} \sum^L_{l=1}f_{\theta^{(l)}}(x); \:\:\: \theta_k^{(l)} = \sigma_k r_k^{(l)} + \mu_k~\forall k
\end{align*}
Note that we apply a reparameterisation trick: By sampling $r_k^{(l)} \sim \mathcal{N}(0,1)$, we move stochasticity away from $\mu_k$ and $\sigma_k$, which allows us to compute their partial derivatives:
\begin{align*}
    \frac{\partial P(y|x,\psi)}{\partial\mu_k} &= \frac{1}{L} \sum^L_{l=1} \frac{\partial}{\partial \theta^{(l)}_k} f_{\theta^{(l)}}(x)~ \frac{\partial \theta^{(l)}_k}{\partial \mu_k} = \frac{1}{L} \sum^L_{l=1} \frac{\partial}{\partial \theta^{(l)}_k} f_{\theta^{(l)}}(x),\\
    \frac{\partial P(y|x,\psi)}{\partial\sigma_k} &= \frac{1}{L} \sum^L_{l=1} \frac{\partial}{\partial \theta^{(l)}_k}f_{\theta^{(l)}}(x)~ \frac{\partial \theta^{(l)}_k}{\partial \sigma_k} = \frac{1}{L} \sum^L_{l=1} \frac{\partial}{\partial \theta^{(l)}_k} f_{\theta^{(l)}}(x)~ r^{(l)}_k\\
\end{align*}
We obtain $\frac{\partial}{\partial \theta^{(l)}_k} f_{\theta^{(l)}}(x)$ by backpropagation.

\subsubsection{\textbf{FIRES And Soft Decision Trees}}
\label{sec:fires_sdt}
Binary decisions as in regular CART Decision Trees are not differentiable. Accordingly, we have to choose a Decision Tree model that has differentiable parameters to compute the gradient of \eqref{eq:gradient_marginal}. One such model is the Soft Decision Tree (SDT) \cite{irsoy2012soft,frosst2017distilling}. Let $n$ be the index of an inner node of the SDT. SDTs replace the binary split at $n$ with a logistic function:
\begin{equation*}
    p_n(x) = \frac{1}{1 + e^{-(\sum^J_{j=1} \theta_{nj} x_j)}}
\end{equation*}
This function yields the probability by which we choose $n$'s right child branch given $x$. Note that the logistic function is differentiable with respect to the parameters $\theta_n$. Similar to ANNs, however, we are now faced with nonlinearity and higher complexity of the model. Therefore, we approximate the marginal likelihood using Monte Carlo and the reparameterisation trick. With $f_\theta$ being an SDT, the partial derivatives of \eqref{eq:gradient_marginal} correspond to those of the ANN.

\subsubsection{\textbf{Aggregating Parameters}}
\label{sec:aggregation}
The number of model parameters $\theta = [\theta_1,..,\theta_k,..,\theta_K]$ might exceed the number of input features $x = [x_1,..,x_j,..,x_J]$, i.e. $K\geq J$. Whenever this is the case, we have to aggregate parameters, since we require a single importance and uncertainty score per input feature to compute feature weights in \eqref{eq:omega}. Next, we show how to aggregate the parameters of an SDT and an ANN to create a meaningful representation.

For SDTs the aggregation is fairly simple. Each inner node is a logistic function that comprises exactly one parameter per input feature. Let $N$ be the number of inner nodes. Accordingly, we have a total of $J \times N$ parameters in the SDT model. We aggregate the parameters associated with an input feature $x_j$ by computing the mean over all inner nodes:
\begin{equation*}
    \theta_j = \frac{1}{N} \sum^N_{n=1} \theta_{nj};\;\;\; \theta_j \sim \mathcal{N}\left(\frac{1}{N}\sum^N_{n=1}\mu_{nj},\frac{1}{N}\sum^N_{n=1}\sigma_{nj}\right)
\end{equation*}
Due to the multi-layer architecture of an ANN, its parameters are usually associated with more than one input feature. For this reason, we cannot apply the same methodology as that proposed for the SDT. Instead, for each input feature $x_j$, we aggregate all parameters that lie along $j$'s path to the output layer. Let $h$ be the index of a layer of the ANN, where $h=1$ denotes the input layer. Let further $\mathcal{U}_h^j$ be the set of all nodes of layer $h$ that belong to the path of $x_j$. We sum up the average parameters along all layers and nodes on $j$'s path to obtain a single aggregated parameter $\theta_j$:
\begin{equation*}
    \theta_j = \sum^{H-1}_{h=1} \frac{1}{|\mathcal{U}_h^j| |\mathcal{U}_{h+1}^j|} \sum_{n\in\mathcal{U}_h^j}\sum_{i\in\mathcal{U}_{h+1}^j} \theta_{ni};\\
\end{equation*}
\begin{align*}
    \theta_j \sim \mathcal{N}\Bigg(&\sum^{H-1}_{h=1} \frac{1}{|\mathcal{U}_h^j| |\mathcal{U}_{h+1}^j|} \sum_{n\in\mathcal{U}_h^j}\sum_{i\in\mathcal{U}_{h+1}^j} \mu_{ni},\\ 
    &\sum^{H-1}_{h=1} \frac{1}{|\mathcal{U}_h^j| |\mathcal{U}_{h+1}^j|} \sum_{n\in\mathcal{U}_h^j}\sum_{i\in\mathcal{U}_{h+1}^j} \sigma_{ni}\Bigg)
\end{align*}
where $H$ is the total number of layers and $|\cdot|$ is the cardinality of a set. Note that $\theta_{ni}$ is the parameter that connects node $n$ in layer $h$ to node $i$ in layer $h+1$.

%--- STABILITY METRIC -----------------%
\subsection{Feature Selection Stability in Data Streams}
\label{sec:stability}
In view of increasing threats such as adversarial attacks, the stability of machine learning models has received much attention in recent years \cite{goodfellow2018making}. Stability usually describes the robustness of a machine learning model against (adversarial or random) perturbations of the data \cite{turney1995bias,bousquet2002stability}. In this context, a feature selection model is considered stable, if the set of selected features does not change after we slightly perturb the input \cite{kalousis2005stability}. To measure feature selection stability, we can therefore monitor the variability of the feature set \cite{nogueira2017stability}. \citet{nogueira2017stability} developed a stability measure, which is a generalization of various existing methods. Let $\mathcal{Z} = [A_1,..,A_r]^T$ be a matrix that contains $r$ feature vectors, which we denote by $A_r \in \{0,1\}^J$. Specifically, $A_r$ is the feature vector that was obtained for the $r$'th sample, such that selected features correspond to $1$ and $0$ otherwise. Feature selection stability according to \citeauthor{nogueira2017stability} \cite{nogueira2017stability} is then defined as:
\begin{equation}
    \Gamma(\mathcal{Z}) = 1 - \frac{\frac{1}{J} \sum^J_{j=1} s^2_j}{\frac{M}{J}\left(1-\frac{M}{J}\right)},
    \label{eq:stability}
\end{equation}
where $M$ is the number of selected features and $s^2_j = \frac{r}{r-1}\hat{p}_j(1-\hat{p}_j)$ is the unbiased sample variance of the selection of feature $j$. Moreover, let $\hat{p}_j = \frac{1}{r}\sum^r_{i=1} z_{ij}$, where $z_{ij}$ denotes one element of $\mathcal{Z}$. According to \eqref{eq:stability}, the feature selection stability decreases, if the total variability $\sum^J_{j=1} s^2_j$ increases. On the other hand, if $s^2_j = 0~\forall j$, i.e. there is no variability in the selected features, the stability reaches its maximum value at $\Gamma(\mathcal{Z}) = 1$.

\citet{nogueira2017stability} show that \eqref{eq:stability} has a clean statistical interpretation. In addition, the measure can be calculated in linear time, making it a sensible choice for evaluating online feature selection models. However, to calculate \eqref{eq:stability}, we would have to sample $r$ feature vectors, which can be costly. A na\"ive but very efficient approach is to use the feature vectors in a shifting window instead. Let $A_t$ be the active feature set at time step $t$. We then define $\mathcal{Z}_t = [A_{t-r+1},...,A_t]^T$, where $r$ depicts the size of the shifting window. We compute the sample variance $s^2_j$ in the same fashion as before. However, \eqref{eq:stability} is now restricted to the observations between time step $t$ and $t-r+1$.

The shifting window approach allows us to update the stability measure for each new feature vector at each time step. However, the approach might return low stability values, when the shifting window falls within a period of concept drift. Accordingly, we should control the sensitivity of the stability measure by a sensible selection of the window size. An alternative to shifting windows are Cross-Validation based schemes recently proposed by \citet{barddal2019boosting}, based on an idea of \citet{bifet2015efficient}. 

Finally, note that \fires\ produces consistent feature rankings when the target distribution is stable (Property \ref{prop:consistency}). In this case, we will ultimately achieve zero variance in subsequent feature sets, thereby maximizing the feature selection stability according to \eqref{eq:stability}.

%--- RELATED WORK -----------------%
\section{Related Work}
\label{sec:related_work}
Traditionally, feature selection models are categorized into filters, wrappers and embedded methods \cite{chandrashekar2014survey,ramirez2017survey,guyon2003introduction}. Embedded methods merge feature selection with the prediction, filter methods are decoupled from the prediction, and wrappers use predictive models to weigh and select features \cite{ramirez2017survey}. Evidently, \fires\ belongs to the group of wrappers.

Data streams are usually defined as an unbounded sequence of observations, where all features are known in advance. This assumption is shared by most related literature. In practice, however, we often observe streaming features, i.e. features that appear successively over time. The approaches dealing with streaming features often assume that only a fixed set of observations is available. We should therefore distinguish between online feature selection methods for ``streaming observations'' and ``streaming features''. Next, we present prominent and recent works in both categories.

\textbf{Streaming Features:} \citet{zhou2006streamwise} proposed a model that selects features based on a potential reduction of error. The threshold used for feature selection is updated with a penalty method called alpha-investing. Later, \citet{wu2012online} introduced \emph{Online Streaming Feature Selection (OSFS)}, which constructs the Markov blanket of a class and gradually removes irrelevant or redundant features. Likewise, the \emph{Scalable and Accurate Online Approach (SAOLA)} removes redundant features by computing a lower bound on pairwise feature correlations \cite{yu2014towards}. Another approach is \emph{Group Feature Selection from Feature Streams (GFSFS)}, which uncovers and removes all feature groups that have a low mutual information with the target variable \cite{li2013online}. Finally, \citet{wang2018provable} introduced \emph{Online Leverage Scores for Feature Selection}, a model that selects features based on the approximate statistical leverage score.

\textbf{Streaming Observations:} An early approach is \emph{Grafting}, which uses gradient updates to iteratively adjust feature weights \cite{perkins2003grafting}. Later, \citet{nguyen2012heterogeneous} employed an ensemble model called \emph{Heterogeneous Ensemble with Feature Drift for Data Streams (HEFT)} to compute a symmetrical uncertainty measure that can be used to weigh and select features. \emph{Online Feature Selection (OFS)} is another wrapper that adjusts feature weights based on misclassifications of a Perceptron \cite{wang2013online}. OFS truncates the weight vector at each time step to retain only the top features. The \emph{Feature Selection on Data Streams (FSDS)} model by \citet{huang2015unsupervised} maintains an approximated low-rank matrix representation of all observed data. FSDS computes feature weights with a Ridge-regression model trained on the low-rank matrix. Recently, \citet{borisov2019cancelout} proposed \emph{Cancel Out}, a sparse layer for feature selection in neural nets, which exploits the gradient information obtained during training. Another recent proposal is the \emph{Adaptive Sparse Confidence-Weighted (ASCW)} model that obtains feature weights from an ensemble of sparse learners \cite{liu2019adaptive}.

Some online feature selection models can process ``streaming features'' and ``streaming observations'' simultaneously. One example is the \emph{Extremal Feature Selection (EFS)} by \citet{carvalho2006single}, which ranks features by the absolute difference between the positive and negative weights of a modified balanced Winnow algorithm. Another approach uses statistical measures like $\chi^2$ to rank and select features \cite{katakis2005utility}. Finally, there are online predictive models that offer embedded feature selection, e.g. \emph{DXMiner} \cite{masud2010classification}, as well as explanation models like \emph{LICON} \cite{kasneci2016licon}, which quantify the influence of input features. For further consultation of related work, we refer the fellow reader to the surveys of \citet{guyon2003introduction}, \citet{li2018feature} and \citet{ramirez2017survey}.

For the evaluation of our framework, we assume that all features are known in advance. Note, however, that \fires\ can also support streaming features by dynamically adding parameters to the likelihood model (we leave a detailed analysis for future work).

%--- EXPERIMENT -----------------%
\section{Experiments}
\label{sec:experiments}
Next, we evaluate the \fires\ framework in multiple experiments. Specifically, we compare the three instantiations of \fires\ introduced above (see Sections \ref{sec:fires_probit} to \ref{sec:fires_sdt}). Besides, we compare our framework to three state-of-the-art online feature selection models: OFS \cite{wang2013online}, FSDS \cite{huang2015unsupervised} and EFS \cite{carvalho2006single}. The hyperparameters of each related model were selected as specified in the corresponding papers. We have also defined a set of default hyperparameters for the three \fires\ models. Details about the hyperparameter search can be found in the Appendix. We chose a binary classification context for the evaluation, as it is a basic problem that all models should handle well.

\subsection{Data Sets}
\begin{table}[t]
\caption{\emph{Data sets}. ``Types'' denotes the data types included in the data set (continuous, categorical). The synthetic RBF (Radial Basis Function) and RTG (Random Tree Generator) data sets were generated with scikit-multiflow \cite{montiel2018scikit}.}
    \label{tab:datasets}
    \centering
    \begin{adjustbox}{max width=\columnwidth}
        \begin{tabular}{llll}
        \toprule
        \textbf{Name}              & \textbf{\#Samples} & \textbf{\#Features} & \textbf{Types} \\ 
        \cmidrule(lr){1-1} \cmidrule(lr){2-4}
        HAR (one-vs-all) & 10,299            & 562               & cont.         \\ 
        Spambase                & 4,601            & 57                & cont.         \\ 
        Usenet (20 Newsgroups)  & 5,931            & 658               & cat.          \\ 
        Gisette ('03 NIPS challenge)  & 7,000            & 5,000               & cont.          \\ 
        Madelon ('03 NIPS challenge)   & 2,600            & 500               & cont.          \\
        Dota   & 100,000            & 116               & cat.          \\
        KDD Cup 1999 Data               & 100,000           & 41                & cont., cat.   \\
        MNIST (one-vs-all)                   & 70,000            & 784               & cont.          \\ 
        RBF (synthetic)       & 10,000           & 10,000               & cont.         \\ 
        RTG (synthetic) & 10,000           & 450                & cont., cat.   \\ 
        \bottomrule
        \end{tabular}
    \end{adjustbox}
\end{table}
We have limited ourselves to known data sets that are either generated by streaming applications or are closely related to them. Table \ref{tab:datasets} shows the properties of all data sets. Further information about the data sets and the preprocessing is included in the Appendix.

The Human Activity Recognition (HAR) and the MNIST data set are multiclass data sets. We have transferred both data sets into a binary classification setting as follows: HAR contains measurements of a motion sensor. The labels denote corresponding activities. For our evaluation, we treated HAR as a one-vs-all classification of the activity ``Walking''. MNIST is a popular digit recognition data set. Here, we chose the label ``3'' for a one-vs-all classification.

The Gisette and Madelon data sets were introduced as part of a NIPS feature selection challenge. We also used a data set from the 1999 KDD Cup that describes fraudulent and benign network activity. The Dota data set contains (won/lost) results of the strategy game Dota 2. Its features correspond to player information, such as rank or game character. Since KDD and Dota are fairly large data sets, we took a random sample of 100,000 observations each to compute the experiments in a reasonable time. Finally, Usenet is a streaming adaptation of the 20 Newsgroups data set and Spambase contains the specifications of various spam and non-spam emails.

Moreover, we generated two synthetic data sets with scikit-multiflow \cite{montiel2018scikit}. Specifically, we obtained 10,000 instances with 10,000 continuous features from the Random RBF Generator (RBF = Radial Basis Function). In addition, we generated another 10,000 instances with 50 categorical and 200 continuous features using the Random Tree Generator (RTG). Here, each categorical feature has five unique, one-hot-encoded values, giving a total of 450 features. We used the default hyperparameters of both generators and defined a random state for reproducibility.

%-------------- MNIST Plot ---------------%
\begin{figure}[t]
\centering
    \begin{subfigure}[t]{0.3\columnwidth}
     \centering
     \includegraphics[width=\textwidth]{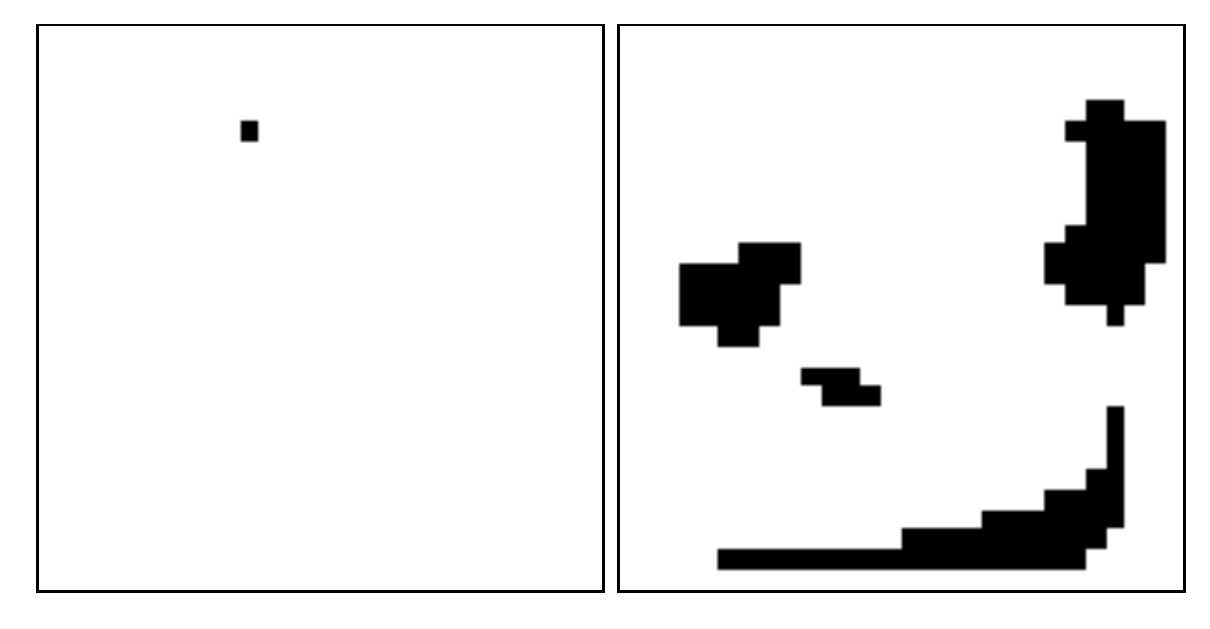}
     \caption{EFS\cite{carvalho2006single}, 3 vs. all}
    \end{subfigure}
    \hfill
    \begin{subfigure}[t]{0.3\columnwidth}
     \centering
     \includegraphics[width=\textwidth]{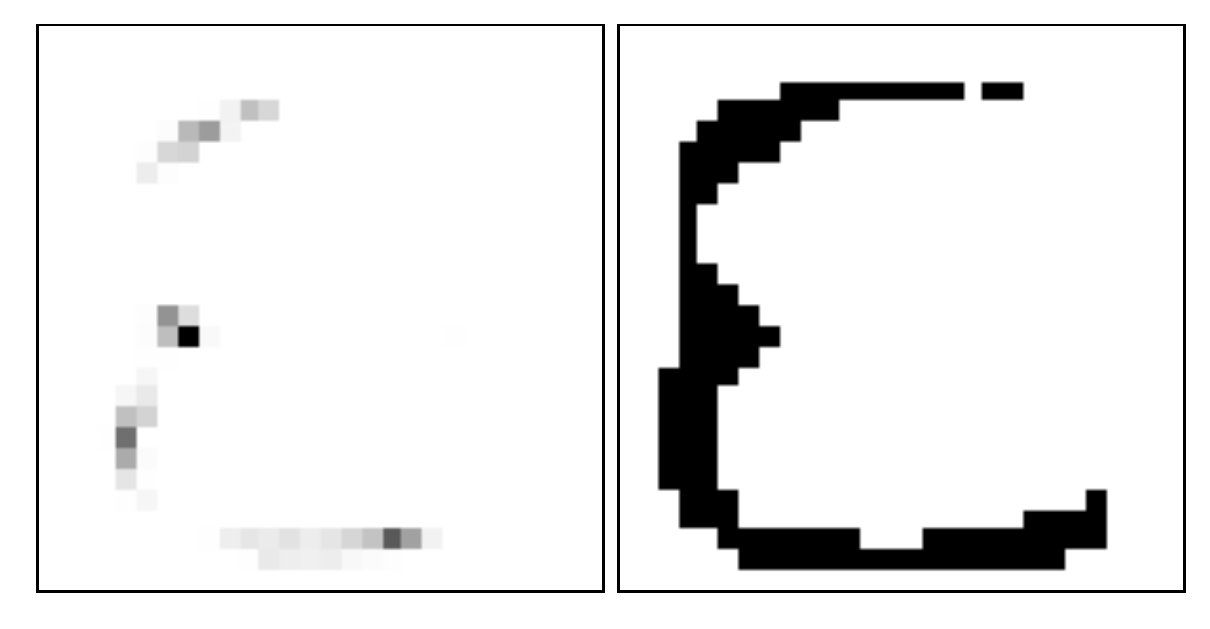}
     \caption{EFS\cite{carvalho2006single}, 8 vs. all}
    \end{subfigure}
    \hfill
    \begin{subfigure}[t]{0.3\columnwidth}
     \centering
     \includegraphics[width=\textwidth]{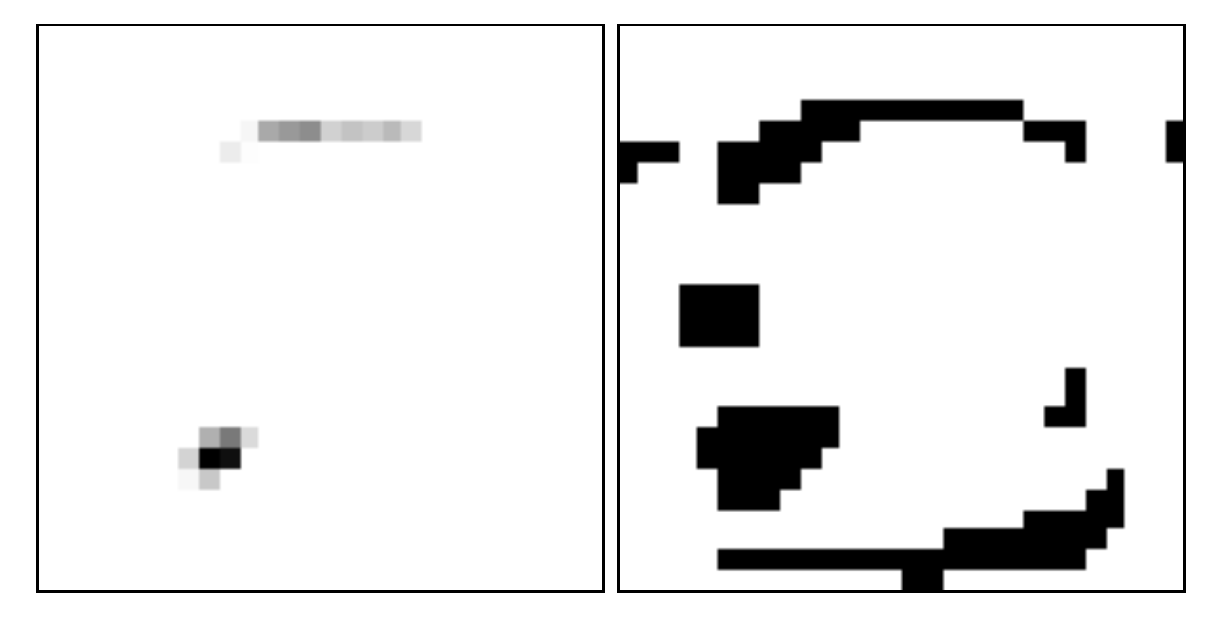}
     \caption{EFS\cite{carvalho2006single}, 9 vs. all}
    \end{subfigure}

    \begin{subfigure}[t]{0.3\columnwidth}
     \centering
     \includegraphics[width=\textwidth]{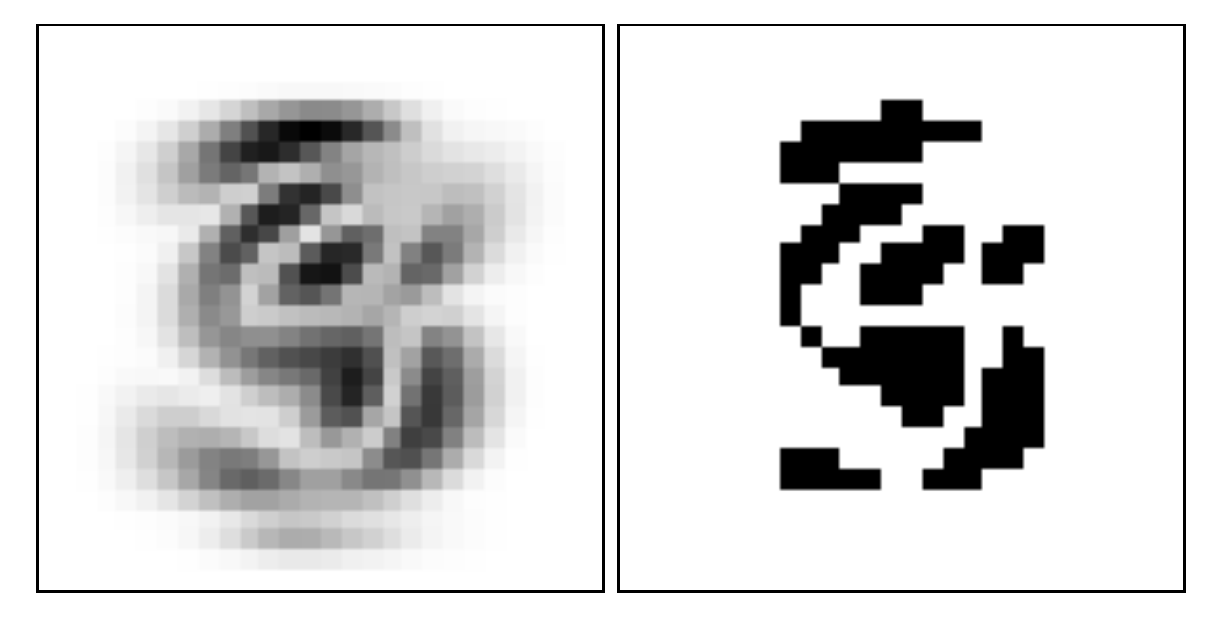}
     \caption{FSDS\cite{huang2015unsupervised}, 3 vs. all}
    \end{subfigure}
    \hfill
    \begin{subfigure}[t]{0.3\columnwidth}
     \centering
     \includegraphics[width=\textwidth]{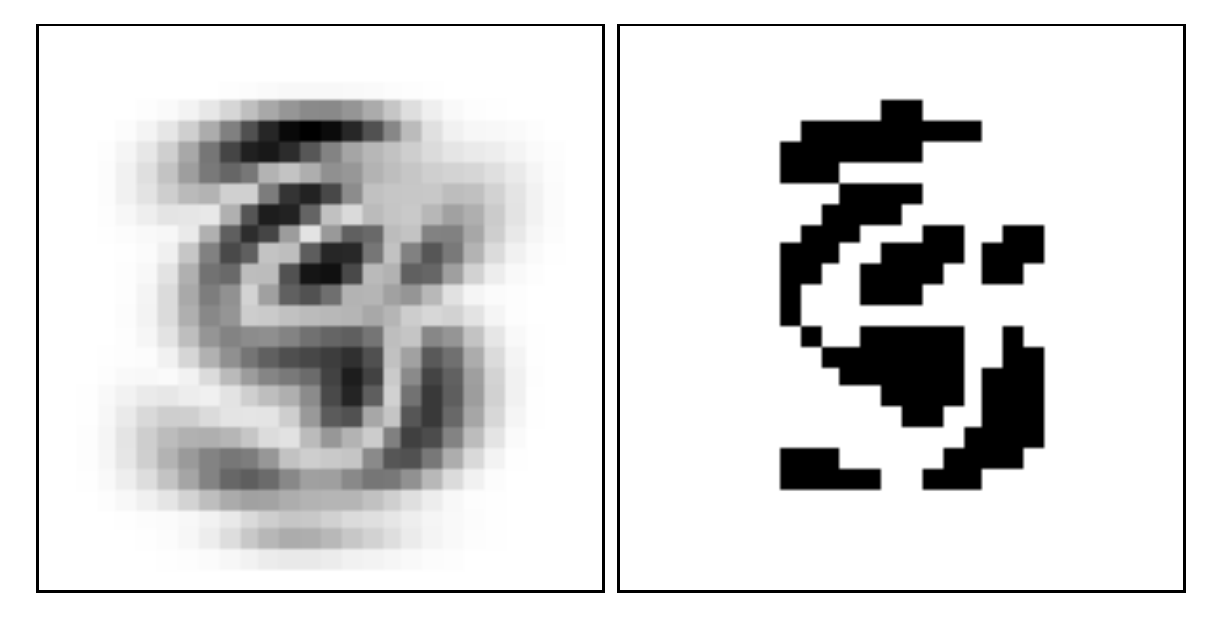}
     \caption{FSDS\cite{huang2015unsupervised}, 8 vs. all}
    \end{subfigure}
    \hfill
    \begin{subfigure}[t]{0.3\columnwidth}
     \centering
     \includegraphics[width=\textwidth]{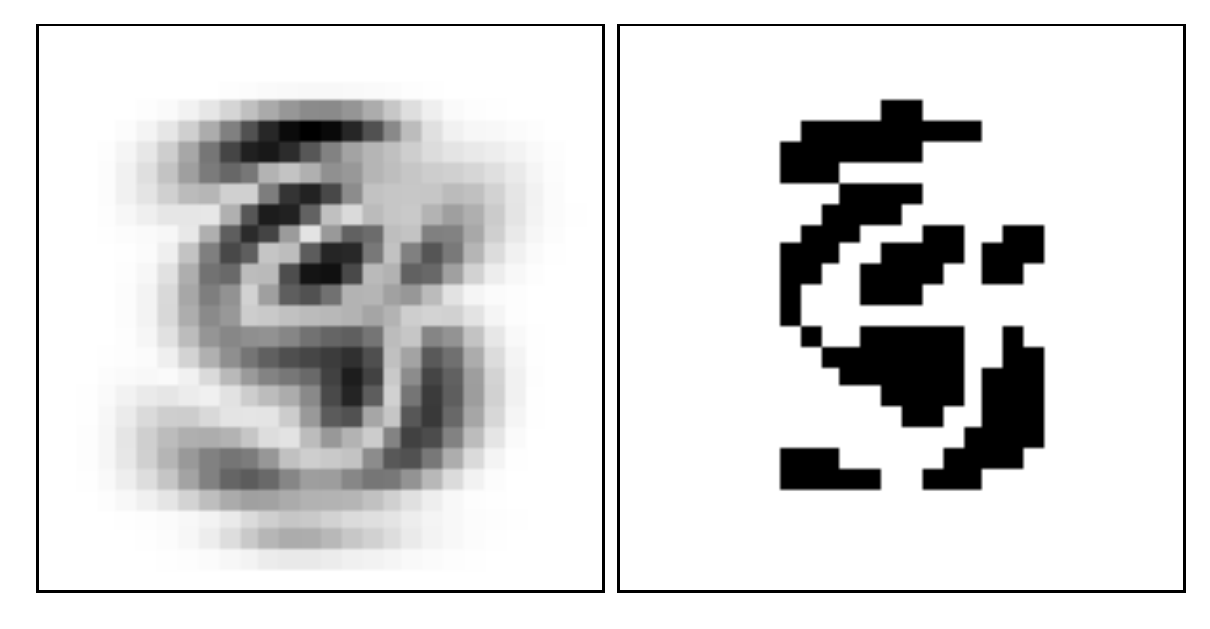}
     \caption{FSDS\cite{huang2015unsupervised}, 9 vs. all}
    \end{subfigure}
    
    \begin{subfigure}[t]{0.3\columnwidth}
     \centering
     \includegraphics[width=\textwidth]{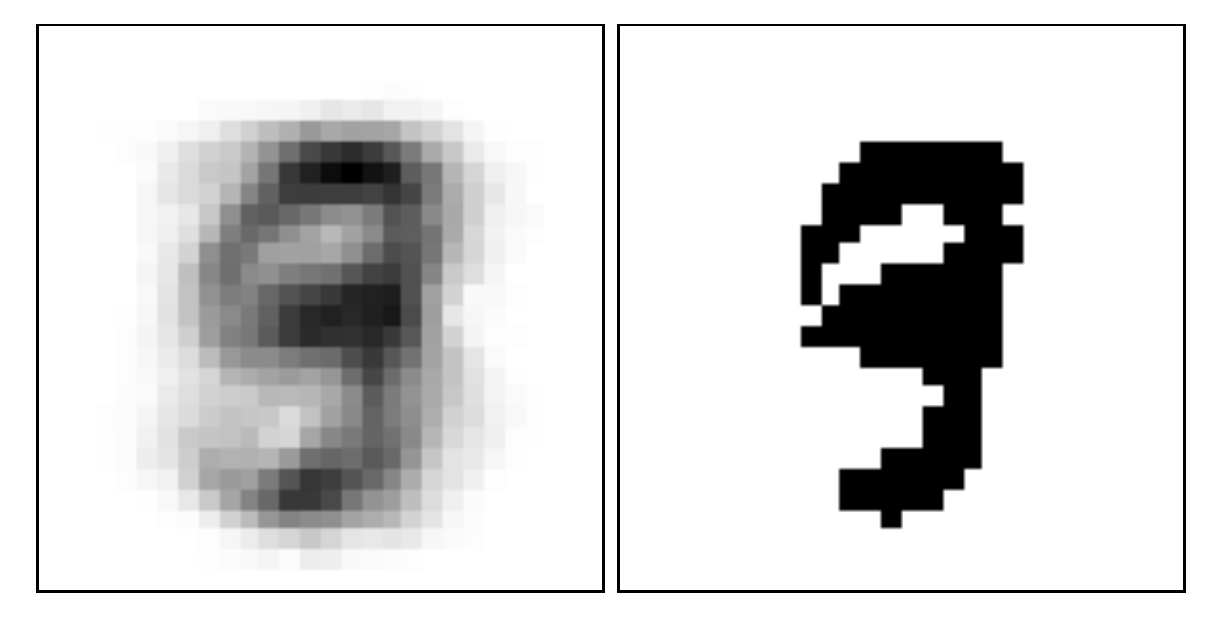}
     \caption{OFS\cite{wang2013online}, 3 vs. all}
    \end{subfigure}
    \hfill
    \begin{subfigure}[t]{0.3\columnwidth}
     \centering
     \includegraphics[width=\textwidth]{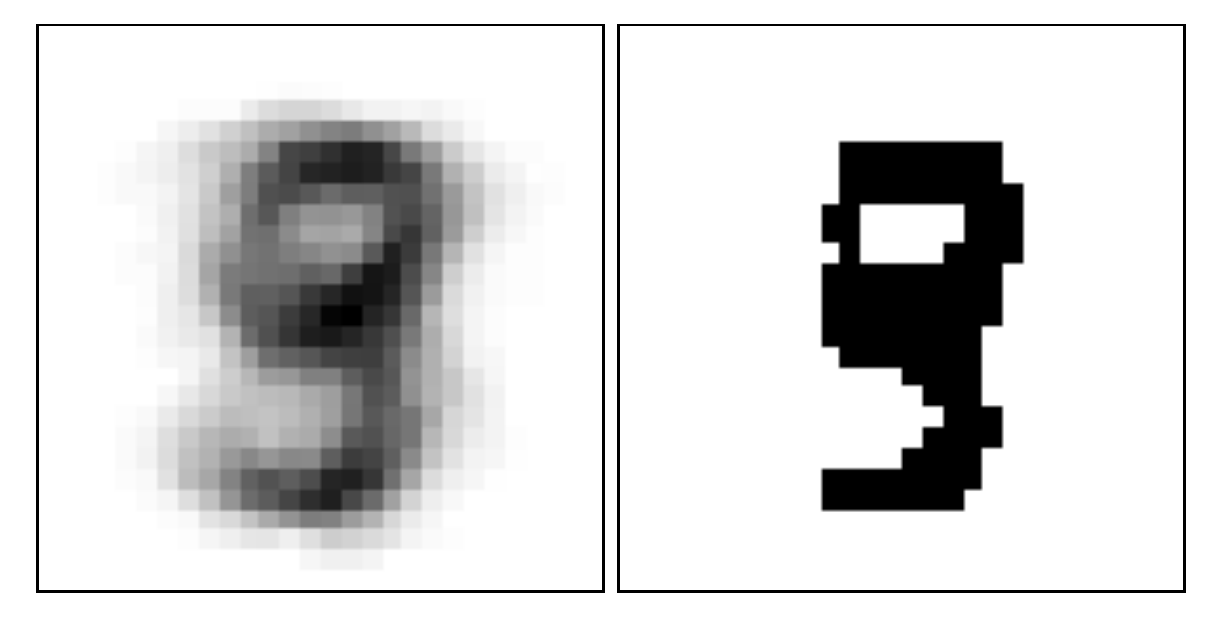}
     \caption{OFS\cite{wang2013online}, 8 vs. all}
    \end{subfigure}
    \hfill
    \begin{subfigure}[t]{0.3\columnwidth}
     \centering
     \includegraphics[width=\textwidth]{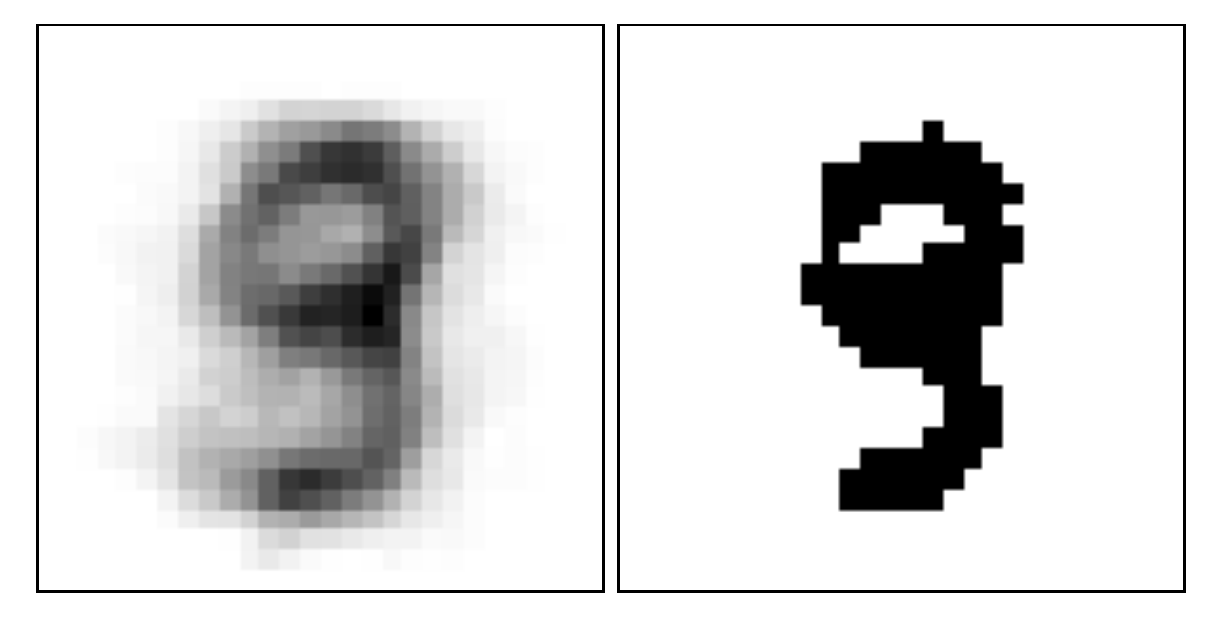}
     \caption{OFS\cite{wang2013online}, 9 vs. all}
    \end{subfigure}
    
    \begin{subfigure}[t]{0.3\columnwidth}
     \centering
     \includegraphics[width=\textwidth]{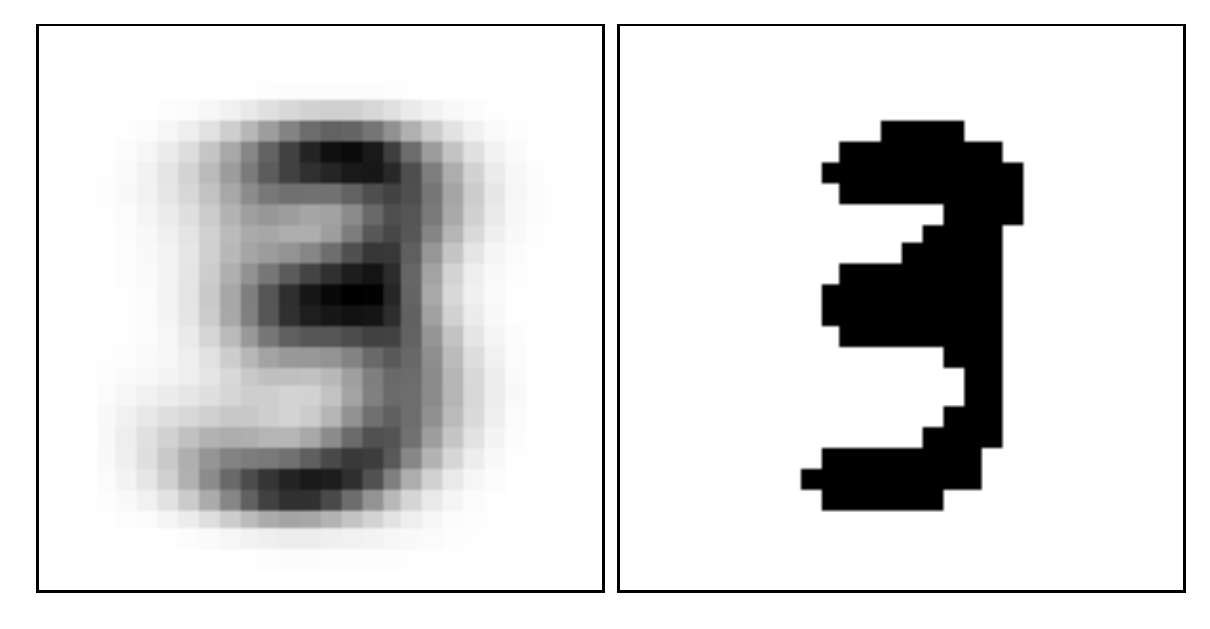}
     \caption{FIRES, 3 vs. all}
    \end{subfigure}
    \hfill
    \begin{subfigure}[t]{0.3\columnwidth}
     \centering
     \includegraphics[width=\textwidth]{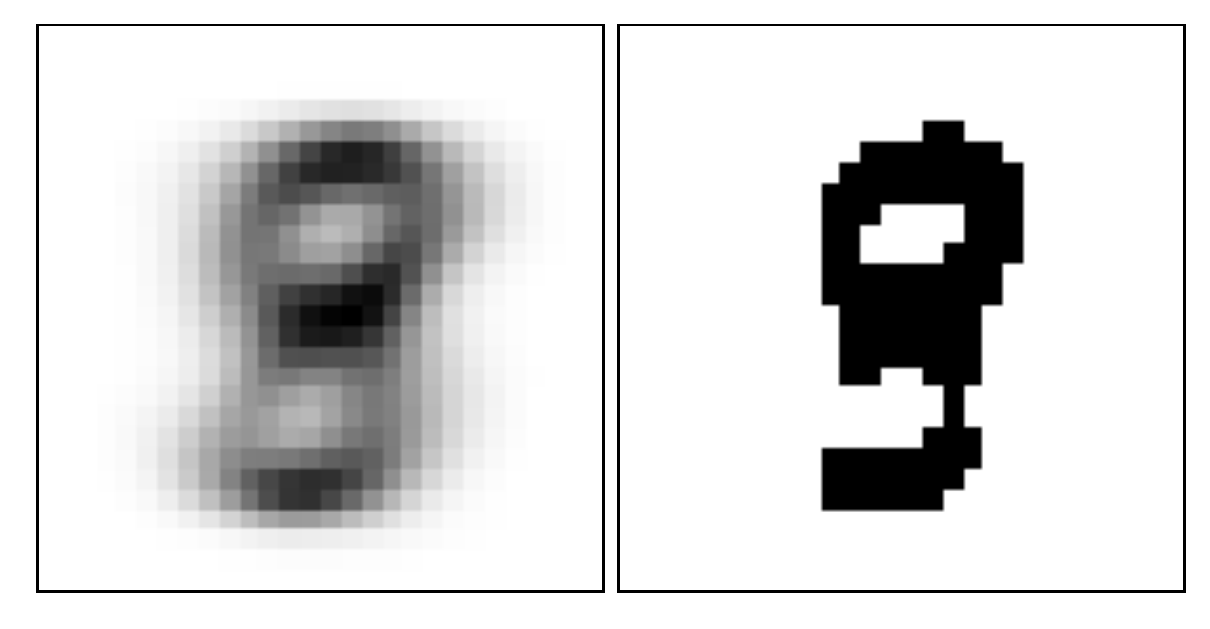}
     \caption{FIRES, 8 vs. all}
    \end{subfigure}
    \hfill
    \begin{subfigure}[t]{0.3\columnwidth}
     \centering
     \includegraphics[width=\textwidth]{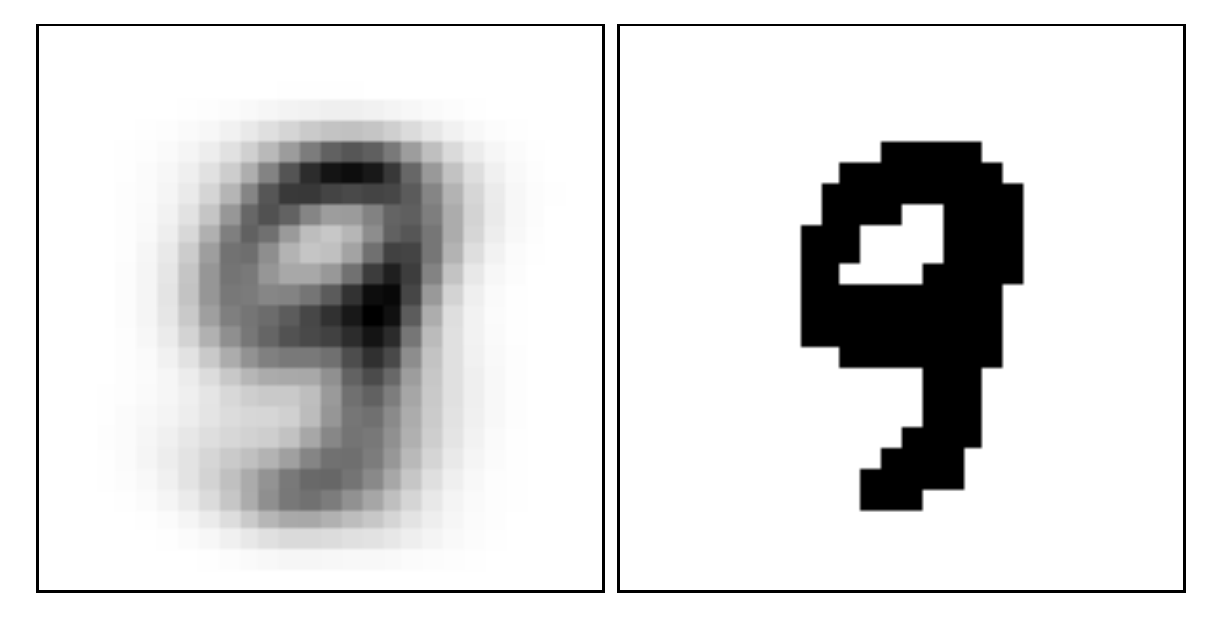}
     \caption{FIRES, 9 vs. all}
    \end{subfigure}
    \caption{\emph{Discriminative Features.} We have sampled all observations from MNIST with the label $3, 8$ or $9$. We successively selected each label as the positive class and carried out a feature selection. The subplots above illustrate the feature weights (left, high weights being dark) and selected features (right, $M=115$) after observing the first 1000 instances. While OFS, FSDS and EFS struggle to select meaningful features, \fires\ (ours) effectively identifies the discriminative features of the positive class after observing very little data.}
    \label{fig:mnist_exp}
\end{figure}
%-------------- MNIST Plot ---------------%

%------------ Result tables ----------%
\begin{table*}[t]
\caption{\emph{Accuracy and Computation Time.} We tested our framework (\fires) with three different base models; a GLM, an ANN and an SDT (Sections \ref{sec:fires_probit} - \ref{sec:fires_sdt}). The Online Boosting \cite{wang2016online} and Very Fast Decision Tree (VFDT) \cite{domingos2000mining} models were trained on the full feature set to serve as a benchmark. All models are evaluated on an i7-8550U CPU with 16 Gb RAM, running Windows 10. Further details about the experimental setup can be found in the Appendix. Here, we show the accuracy (acc.) and computation time (feature selection + model training; milliseconds) observed on average per time step. Strikingly, while all three \fires\ models are competitive, \emph{FIRES-GLM} takes first place on average in terms of both predictive accuracy and computation time.}
    \label{tab:accuracy}
    \centering
    \begin{adjustbox}{max width=\textwidth}
        \begin{tabular}{lllllllllllllllll}
        \toprule
        & \multicolumn{4}{c}{Benchmark Models} & \multicolumn{12}{c}{Feature Selection Models (with Perceptron)}\\
        \cmidrule(lr){2-5}\cmidrule(lr){6-17}
        & \multicolumn{2}{c}{Boosting \cite{wang2016online}} & \multicolumn{2}{c}{VFDT \cite{domingos2000mining}} & \multicolumn{2}{c}{FIRES-GLM} & \multicolumn{2}{c}{FIRES-ANN} & \multicolumn{2}{c}{FIRES-SDT} & \multicolumn{2}{c}{OFS \cite{wang2013online}} & \multicolumn{2}{c}{EFS \cite{carvalho2006single}} & \multicolumn{2}{c}{FSDS \cite{huang2015unsupervised}}\\
        \cmidrule(lr){2-3}\cmidrule(lr){4-5}\cmidrule(lr){6-7}\cmidrule(lr){8-9}\cmidrule(lr){10-11}\cmidrule(lr){12-13}\cmidrule(lr){14-15}\cmidrule(lr){16-17}
        Datasets & acc.& ms & acc. & ms & acc. & ms & acc. & ms & acc. & ms & acc. & ms & acc. & ms & acc. & ms \\
        \cmidrule(lr){1-1}\cmidrule(lr){2-3}\cmidrule(lr){4-5}\cmidrule(lr){6-7}\cmidrule(lr){8-9}\cmidrule(lr){10-11}\cmidrule(lr){12-13}\cmidrule(lr){14-15}\cmidrule(lr){16-17}
        %--- DATA ---%
        HAR & 0.788 & 2264.805 & 0.819 & 192.515 & 0.87 & \textbf{2.729} & 0.842 & 38.156 & 0.872 & 27.497 & 0.928 & 20.892 & 0.87 & 43.309 & \textbf{0.919} & 4.195\\
        Spambase & \textbf{0.83} & 234.741 & 0.638 & 21.096 & 0.742 & \textbf{1.761} & 0.685 & 28.341 & 0.66 & 23.333 & 0.721 & 10.135 & 0.577 & 15.463 & 0.657 & 1.969\\
        Usenet & 0.507 & 2448.95 & 0.504 & 206.354 & 0.556 & \textbf{2.996} & 0.541 & 41.897 & 0.541 & 27.948 & \textbf{0.563} & 44.696 & 0.531 & 99.758 & 0.537 & 4.331\\
        Gisette & 0.673 & 15026.415 & 0.687 & 1781.348 & \textbf{0.933} & \textbf{22.391} & 0.902 & 164.682 & 0.906 & 42.29 & 0.911 & 291.428 & 0.881 & 667.923 & 0.921 & 36.806\\
        Madelon & \textbf{0.551} & 1398.29 & 0.481 & 245.073 & 0.523 & \textbf{2.849} & 0.509 & 38.992 & 0.508 & 27.482 & 0.522 & 51.872 & 0.511 & 82.638 & 0.539 & 4.217\\
        Dota & \textbf{0.552} & 330.895 & 0.521 & 54.143 & 0.516 & \textbf{2.200} & 0.505 & 30.645 & 0.505 & 23.62 & 0.514 & 18.084 & 0.504 & 24.396 & 0.505 & 2.944\\
        KDD & \textbf{0.989} & 110.659 & 0.985 & 12.226 & 0.969 & 2.073 & 0.928 & 28.614 & 0.956 & 22.96 & 0.962 & 1.829 & 0.971 & 10.756 & 0.785 & \textbf{1.977} \\
        MNIST & 0.906 & 1960.237 & 0.894 & 275.872 & 0.930 & \textbf{3.298} & 0.884 & 43.129 & 0.940 & 28.607 & \textbf{0.950} & 14.597 & 0.879 & 24.094 & 0.930 & 6.075\\
        RBF & 0.323 & 26265.709 & 0.738 & 3567.720 & 0.973 & \textbf{34.953} & \textbf{0.995} & 226.017 & 0.973 & 38.409 & 0.748 & 1356.477 & 0.973 & 1594.846 & 0.984 & 85.514\\
        RTG & 0.812 & 1244.723 & \textbf{0.835} & 203.776 & 0.793 & \textbf{2.609} & 0.730 & 35.698 & 0.743 & 26.166 & 0.743 & 31.511 & 0.789 & 45.706 & 0.719 & 4.165\\
        \cmidrule(lr){1-1}\cmidrule(lr){2-3}\cmidrule(lr){4-5}\cmidrule(lr){6-7}\cmidrule(lr){8-9}\cmidrule(lr){10-11}\cmidrule(lr){12-13}\cmidrule(lr){14-15}\cmidrule(lr){16-17}
        Mean & 0.693 & 5128.542 & 0.710 & 656.012 & \textbf{0.781} & \textbf{7.786} & 0.752 & 67.617 & 0.760 & 28.831 & 0.756 & 184.152 & 0.749 & 260.889 & 0.750 & 15.219\\
        Rank & 8. & 8. & 7. & 7. & \textbf{1.} & \textbf{1.} & 4. & 4. & 2. & 3. & 3. & 5. & 6. & 6. & 5. & 2.\\
        \bottomrule
        \end{tabular}
    \end{adjustbox}
\end{table*}

\begin{table}[t]
\caption{\emph{Feature Selection Stability.} Here, we show the average feature selection stability per time step according to \eqref{eq:stability}. The size of the shifting window was 10. All \fires~models produce consistently stable feature sets, with the GLM based model ranking first place on average.}
    \label{tab:stability}
    \centering
    \begin{adjustbox}{max width=\columnwidth}
        \begin{tabular}{lllllllll}
        \toprule
        & \multicolumn{6}{c}{Feature Selection Models}\\
        \cmidrule(lr){1-1}\cmidrule(lr){2-7}
        & FIRES & FIRES & FIRES & & &\\
        Datasets & -GLM & -ANN & -SDT & OFS \cite{wang2013online} & EFS \cite{carvalho2006single} & FSDS \cite{huang2015unsupervised}\\
        \cmidrule(lr){1-1}\cmidrule(lr){2-7}
        %--- DATA ---%
        HAR & 0.985 & 0.652 & 0.865 & 0.756 & 0.921 & \textbf{0.986} \\
        Spambase & 0.901 & 0.819 & 0.710 & \textbf{0.971} & 0.822 & 0.908 \\
        Usenet & 0.820 & 0.931 & 0.747 & 0.775 & 0.749 & \textbf{0.937} \\
        Gisette & 0.937 & \textbf{0.987} & 0.756 & 0.295 & 0.845 & 0.949 \\
        Madelon & 0.526 & 0.489 & 0.682 & 0.158 & \textbf{0.783} & 0.281 \\
        Dota & 0.978 & 0.950 & 0.932 & 0.444 & \textbf{0.993} & 0.700 \\
        KDD & 0.997 & 0.996 & 0.978 & 0.940 & 0.990 & \textbf{0.999} \\
        MNIST & \textbf{0.996} & 0.753 & 0.950 & 0.703 & 0.981 & 0.989 \\
        RBF & 0.959 & \textbf{0.993} & 0.793 & 0.018 & 0.906 & 0.812 \\
        RTG & \textbf{0.856} & 0.582 & 0.827 & 0.457 & 0.840 & 0.080 \\
        \cmidrule(lr){1-7}
        Mean (Rank) & \textbf{0.896} & 0.815 & 0.824 & 0.552 & 0.883 & 0.764\\
        Rank & \textbf{1.} & 3. & 4. & 6. & 2. & 5.\\
        \bottomrule
        \end{tabular}
    \end{adjustbox}
\end{table}
%------------ Result tables ----------%

\subsection{Results}
Feature selection generally aims to identify input patterns that are discriminative with respect to the target. We show that \fires\ does indeed identify discriminative features by selecting the MNIST data set for illustration. The task was to distinguish the class labels $3$, $8$ and $9$, which can be difficult due to their similarity. By successively using each class as the true label, we obtained three different feature weights, which are shown in Figure \ref{fig:mnist_exp}. In this experiment, we used \fires\ with the ANN base model. Strikingly, while all other models had difficulty in selecting a meaningful set of features, \fires\ effectively captured the true pattern of the positive class. In fact, \fires\ has produced feature weights and feature sets that are easy for humans to interpret. %Meanwhile, the EFS model did not provide meaningful feature weights. This could be due to the fact that the underlying Winnow model is not able to represent the problem of digit recognition sufficiently well.

%------------ Acc/Stab development ----------%
\begin{figure}[t]
\centering
     \begin{subfigure}[t]{0.93\columnwidth}
         \centering
         \includegraphics[width=\textwidth]{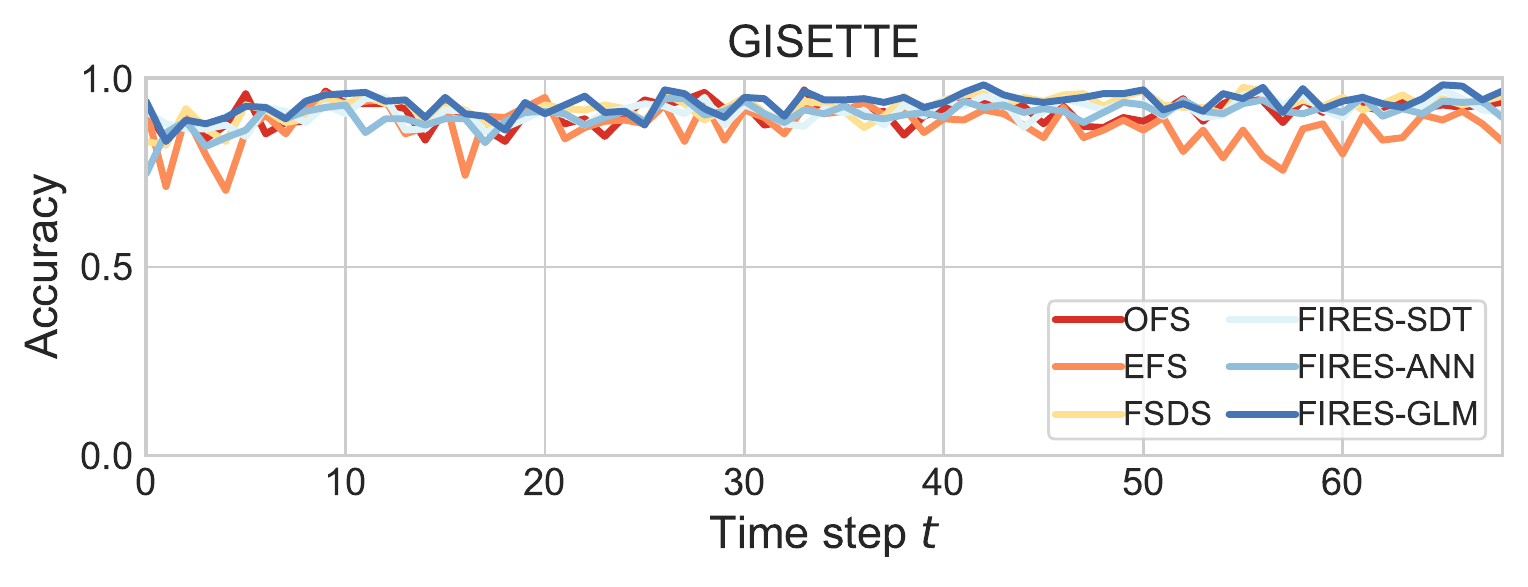}
     \end{subfigure}
     \hfill
     \begin{subfigure}[t]{0.93\columnwidth}
         \centering
         \includegraphics[width=\textwidth]{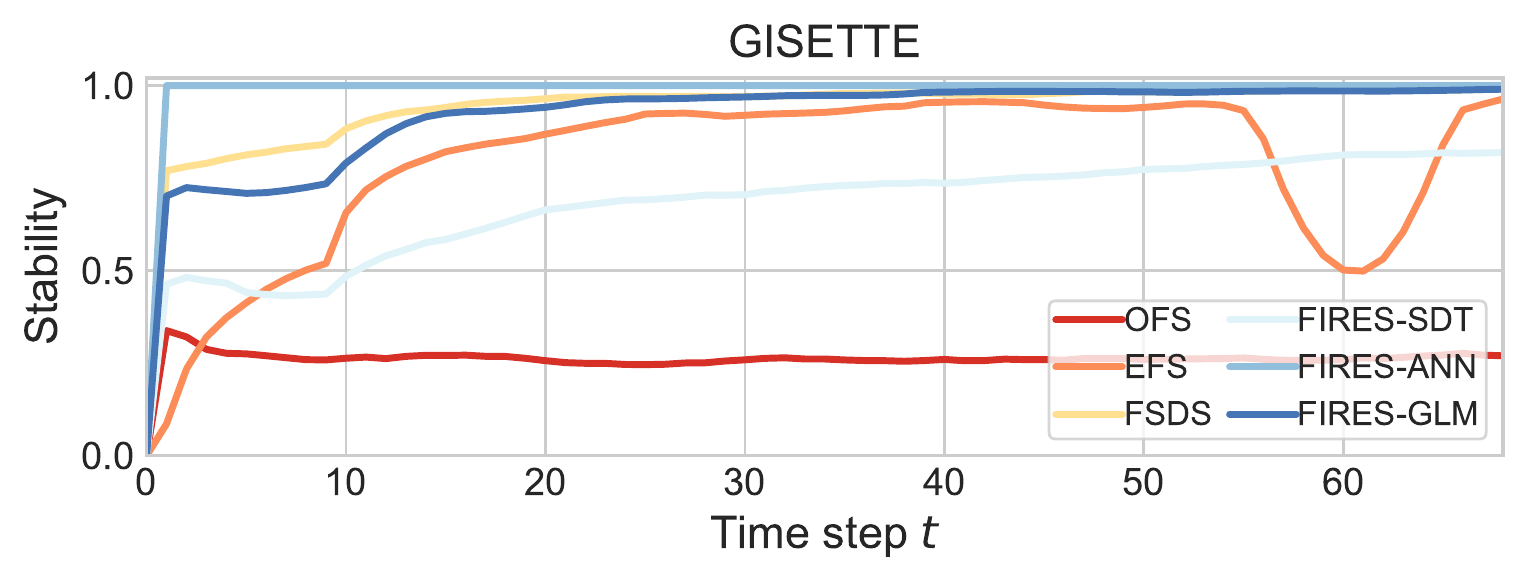}
     \end{subfigure}
    \caption{\emph{Accuracy and Stability over Time.} Here, we show how the accuracy and stability scores develop over time. We used the Gisette data for illustration. We trained on batches of size 100 and used a shifting window of size 10 to compute stability (Eq. \eqref{eq:stability}). All feature selection models achieve similar accuracy. The \fires\ models maximize stability over time, which was also observed for the remaining data sets.}
    \label{fig:acc_develop}
\end{figure}
%------------ Acc/Stab development ----------%

Note that all models we compare are wrapper methods \cite{ramirez2017survey}. As such, they use a predictive model for feature selection, but do not make predictions themselves. To assess the predictive power of feature sets, we therefore had to choose an online classifier that was trained on the selected features. We chose a Perceptron algorithm because it is relatively simple, but still impressively demonstrates the positive effect of online feature selection. As Table \ref{tab:accuracy} shows, feature selection can increase the performance of a Perceptron to the point where it can compete with state-of-the-art online predictive models. For comparison, we trained an Online Boosting \cite{wang2016online} model and a Very Fast Decision Tree (VFDT) \cite{domingos2000mining} on the full feature space. All predictive models were trained in an interleaved test-then-train fashion (i.e. prequential evaluation). The Online Boosting model shows very poor performance for the RBF data set. This is due to the fact that the Na\"ive Bayes models, which we used as base learners, were unable to enumerate the high-dimensional RBF data set, given the relatively small sample size. Nevertheless, we have kept the same boosting architecture in all experiments for reasons of comparability. We chose accuracy as a prediction metric, because it is a common choice in the literature. Moreover, since we do not take into account extremely imbalanced data, accuracy provides a meaningful assessment of the discriminative power of each model. 

Table \ref{tab:accuracy} and Table \ref{tab:stability} exhibit the average accuracy, computation time and stability of multiple evaluations with varying batch sizes (25,50,75,100) and fractions of selected features (0.1, 0.15, 0.2). In Figure \ref{fig:acc_develop} we show how the accuracy and stability develops over time. In addition, Figure \ref{fig:auc_stab} illustrates how the different models manage to balance accuracy and stability. We selected Gisette for illustration, because it is a common benchmark data set for feature selection. Note, however, that we have observed similar results for all remaining data sets.

The results show that our framework is among the top models for online feature selection in terms of computation time, predictive accuracy and stability. Strikingly, \fires\ coupled with the least complex base model (GLM) takes first place on average in all three categories (see Table \ref{tab:accuracy} and \ref{tab:stability}). The GLM based model generates discriminative feature sets, even if the data is not linearly separable (e.g. MNIST), which may seem strange at first. For feature weighting, however, it is sufficient if we capture the relative importance of the input features in the prediction. Therefore, a model does not necessarily have to have a high classification accuracy. Since GLMs only have to learn a few parameters, they tend to recognize important features faster than other, more complex models. Furthermore, \emph{FIRES-ANN} and \emph{FIRES-SDT} are subject to uncertainty due to the sampling we use to approximate the marginal likelihood. In this experiment, \emph{FIRES-GLM} achieved better results than all related models. Still, whenever a linear model may not be sufficient, the \fires\ framework is flexible enough to allow base models of higher complexity.

%------------ Acc/Stab plots ----------%
\begin{figure}[t]
\centering
    \includegraphics[width=0.93\columnwidth]{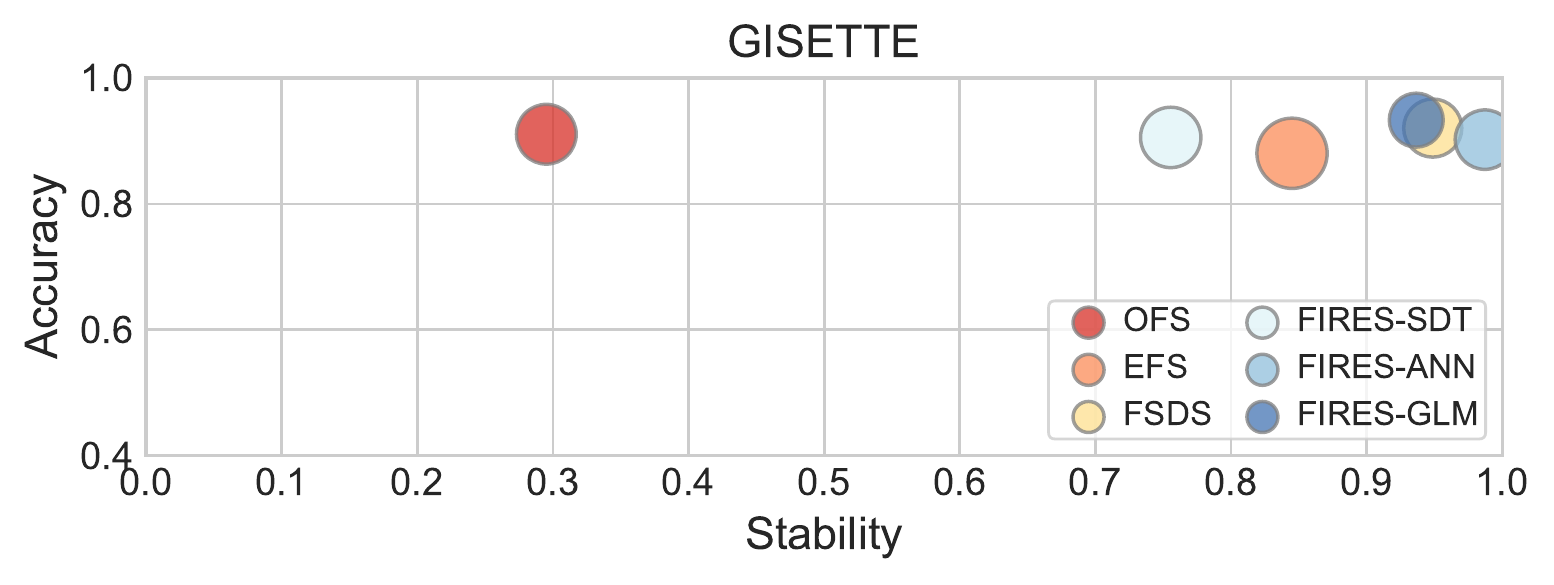}
    \caption{\emph{Accuracy vs. Stability.} In practice, predictions must be both robust and accurate. The proposed framework (\fires) aims at maximizing both aspects and produces feature sets accordingly, which we exemplary show for the Gisette data set. The circle size indicates the variance in accuracy. A perfect result would lie in the top right corner of the plot. Note that in all data sets, we always found at least one \fires-model in the top region (here, it is the models \emph{FIRES-GLM} and \emph{FIRES-ANN}).}
    \label{fig:auc_stab}
\end{figure}
%------------ Acc/Stab plots ----------%

%--- CONCLUSION -----------------%
\section{Conclusion}
In this work, we introduced a novel feature selection framework for high-dimensional streaming applications, called \fires. Using probabilistic principles, our framework extracts importance and uncertainty scores regarding the current predictive power of every input feature. We weigh high importance against uncertainty, producing feature sets that are both discriminative and robust. The proposed framework is modular and can therefore be adapted to the requirements of any learning task. For illustration, we applied \fires\ to three common linear and nonlinear models and evaluated it using several real-world and synthetic data sets. Experiments show that \fires\ produces more stable and discriminative feature sets than state-of-the-art online feature selection approaches, while offering a clear advantage in terms of computational efficiency.

\bibliographystyle{ACM-Reference-Format}
\balance
%\bibliography{bibliography}
%---BIBLIOGRAPHY---------------------------%
%%% -*-BibTeX-*-
%%% Do NOT edit. File created by BibTeX with style
%%% ACM-Reference-Format-Journals [18-Jan-2012].

%---BIBLIOGRAPHY---------------------------%

\appendix
\section{Important Formulae}
\begin{lemma}
\label{app:lemma_1}
	Let $X\sim\mathcal{N}(\mu,\sigma)$ be a normally-distributed random variable and let $\Phi$ be the cumulative distribution function (CDF) of a standard normal distribution. Then the following equation holds:
	\begin{align}
		\mathbb{E}\left[\Phi(\alpha X + \beta)\right] &= \int_{-\infty}^{\infty}\Phi(\alpha x + \beta)P(X = x) ~dx \nonumber\\ 
		&= \Phi\left(\frac{\beta+\alpha\mu}{\sqrt{1 + \alpha^2\sigma^2}}\right)
		\label{eqn:Phi-expectation}
	\end{align}
\end{lemma}
\begin{proof}
Let $Y\sim\mathcal{N}(0,1)$ be a standard-normal-distributed random variable (independent of $X$). We can then rewrite $\Phi$ as: $\Phi(r)=P(Y \leq r)$. Using this, we get:
\begin{align*}
	& \mathbb{E}\left[\Phi(\alpha X + \beta)\right]\\
	={} & \mathbb{E}\left[P(Y \leq \alpha X + \beta)\right]\\
	={} & P(Y \leq \alpha X + \beta)\\
	={} & P(Y - \alpha X \leq \beta)
\end{align*}
The linear combination $Y - \alpha X$ is again normal-distributed with mean $-\alpha \mu$ and variance $\sqrt{1^2 + \alpha^2\sigma^2}$. Hence, we can write
\[
	Y - \alpha X = \sqrt{1 + \alpha^2\sigma^2}Z - \alpha \mu
\]
where $Z\sim\mathcal{N}(0,1)$ is again standard-normal-distributed. We then get:
\begin{align*}
& Y - \alpha X \leq \beta\\
	\Leftrightarrow\quad{} & \sqrt{1 + \alpha^2\sigma^2}Z - \alpha \mu \leq \beta\\
	\Leftrightarrow\quad{} & Z \leq \frac{\beta + \alpha \mu}{\sqrt{1 + \alpha^2\sigma^2}}
\end{align*}
Hence, we get
\begin{align*}
	& \mathbb{E}\left[\Phi(\alpha X + \beta)\right]\\
	={} & P(Y - \alpha X \leq \beta)\\
	={} & P\left(Z \leq \frac{\beta + \alpha \mu}{\sqrt{1 + \alpha^2\sigma^2}}\right)\\
	={} & \Phi\left(\frac{\beta + \alpha \mu}{\sqrt{1 + \alpha^2\sigma^2}}\right)
\end{align*}
\end{proof}

\begin{lemma}
\label{app:lemma_2}
	Let $X=(X_1,...,X_n)$ be normally-distributed random variables with $X_i\sim\mathcal{N}(\mu_i,\sigma_i)$. Let $\alpha_i, \beta\in\mathbb{R}$ be coefficients and let $\Phi$ be the CDF of a standard normal distribution. Then the following equation holds:
	\begin{align}
	    &\mathbb{E}\left[\Phi\left(\sum_{i=1}^n\alpha_i X_i + \beta\right)\right] \nonumber\\ 
		&= \int_{-\infty}^{\infty}\cdot\cdot\int_{-\infty}^{\infty}\Phi\left(\sum_{i=1}^n\alpha_i x_i + \beta\right)\prod_{i=1}^n P(X_i = x_i) ~dx_n\,.\,.\,dx_1 \nonumber\\
		& = \Phi\left(\frac{\beta+\sum_{i=1}^n\alpha_i\mu_i}{\sqrt{1 + \sum_{i=1}^n\alpha_i^2\sigma_i^2}}\right) \label{eqn:multivariate-Phi-expectation}
	\end{align}
\end{lemma}
\begin{proof}
	Proof by mathematical induction over $n$:\\
	{\textbf Base case $n=1$:} Lemma~\ref{app:lemma_1}\\
	{\textbf Inductive step $n\mapsto n+1$:}
	Let (\ref{eqn:multivariate-Phi-expectation}) hold for $n$. For $n+1$, we get:
	\begin{align*}
		& \mathbb{E}\Biggl[\Phi\Biggl(\sum_{i=1}^{n+1}\alpha_i X_i + \beta\Biggr)\Biggr]\\ 
		&= \mathbb{E}\Biggl[\Phi\Biggl(\sum_{i=1}^{n}\alpha_i X_i + \underbrace{\alpha_{n+1} X_{n+1} +\beta}_{=:\tilde\beta}\Biggr)\Biggr]\\
		& \stackrel{(\ref{eqn:multivariate-Phi-expectation})}= \mathbb{E}\Biggl[\Phi\left(\frac{\tilde\beta+\sum_{i=1}^n\alpha_i\mu_i}{\sqrt{1 + \sum_{i=1}^n\alpha_i^2\sigma_i^2}}\right)\Biggr]\\
		& = \mathbb{E}\Biggl[\Phi\left(\frac{\alpha_{n+1}}{\sqrt{1 + \sum_{i=1}^n\alpha_i^2\sigma_i^2}}X_{n+1} + \frac{\beta +\sum_{i=1}^n\alpha_i\mu_i}{\sqrt{1 + \sum_{i=1}^n\alpha_i^2\sigma_i^2}}\right)\Biggr] \\
		& \stackrel{(\ref{eqn:Phi-expectation})}= \Phi\left(\frac{\frac{\beta +\sum_{i=1}^n\alpha_i\mu_i}{\sqrt{1 + \sum_{i=1}^n\alpha_i^2\sigma_i^2}} + \frac{\alpha_{n+1}}{\sqrt{1 + \sum_{i=1}^n\alpha_i^2\sigma_i^2}}\mu_{n+1}}{\sqrt{1 + \left(\frac{\alpha_{n+1}}{\sqrt{1 + \sum_{i=1}^n\alpha_i^2\sigma_i^2}}\right)^2\sigma_{n+1}^2}}\right)\\
		& = \Phi\left(\frac{\beta +\sum_{i=1}^n\alpha_i\mu_i + \alpha_{n+1}\mu_{n+1}}{\sqrt{1 + \sum_{i=1}^n\alpha_i^2\sigma_i^2 + \alpha_{n+1}^2\sigma_{n+1}^2}}\right)\\ 
		&= \Phi\left(\frac{\beta +\sum_{i=1}^{n+1}\alpha_i\mu_i}{\sqrt{1 + \sum_{i=1}^{n+1}\alpha_i^2\sigma_i^2}}\right)
	\end{align*}
\end{proof}

\section{Experimental Setting}
\subsection{Python Packages}
All experiments were conducted on an i7-8550U CPU with 16 Gb RAM, running Windows 10. We have set up an Anaconda (v4.8.2) environment with Python (v3.7.1). We mostly used standard Python packages, including \emph{numpy} (v1.16.1), \emph{pandas} (v0.24.1), \emph{scipy} (v1.2.1), and \emph{scikit-learn} (v0.20.2). The ANN base model was implemented with \emph{pytorch} (v1.0.1). Besides, we used the SDT implementation provided at \url{https://github.com/AaronX121/Soft-Decision-Tree/blob/master/SDT.py}, extracting the gradients after every training step. 

During the evaluation, we further used the \emph{FileStream}, \emph{RandomRBFGenerator}, \emph{RandomTreeGenerator}, \emph{PerceptronMask}, \emph{HoeffdingTree} (VFDT) and \emph{OnlineBoosting} functions of \emph{scikit-multiflow} (v0.4.1). Finally, we generated plots with \emph{matplotlib} (v3.0.2) and \emph{seaborn} (v0.9.0).

\subsection{\emph{scikit-multiflow} Hyperparameters}
If not explicitly specified here, we have used the default hyperparameters of scikit-multiflow \cite{montiel2018scikit}. 

For the \emph{OnlineBoosting()} function, we have specified the following hyperparameters:
\begin{itemize}
    \item \textbf{base\_estimator} = NaiveBayes()
    \item \textbf{n\_estimators} = 3
    \item \textbf{drift\_detection} = False
    \item \textbf{random\_state} = 0
\end{itemize}

Besides, we used the \emph{HoeffdingTree()} function to train a VFDT model. We set the parameter \emph{leaf\_prediction} to \emph{'mc'} (majority class), since the default choice \emph{'nba'} (adaptive Na\"ive Bayes) is very inefficient in high-dimensional applications.

%------------ Lambda plot ----------%
\begin{figure}[t]
\centering
    \includegraphics[width=\columnwidth]{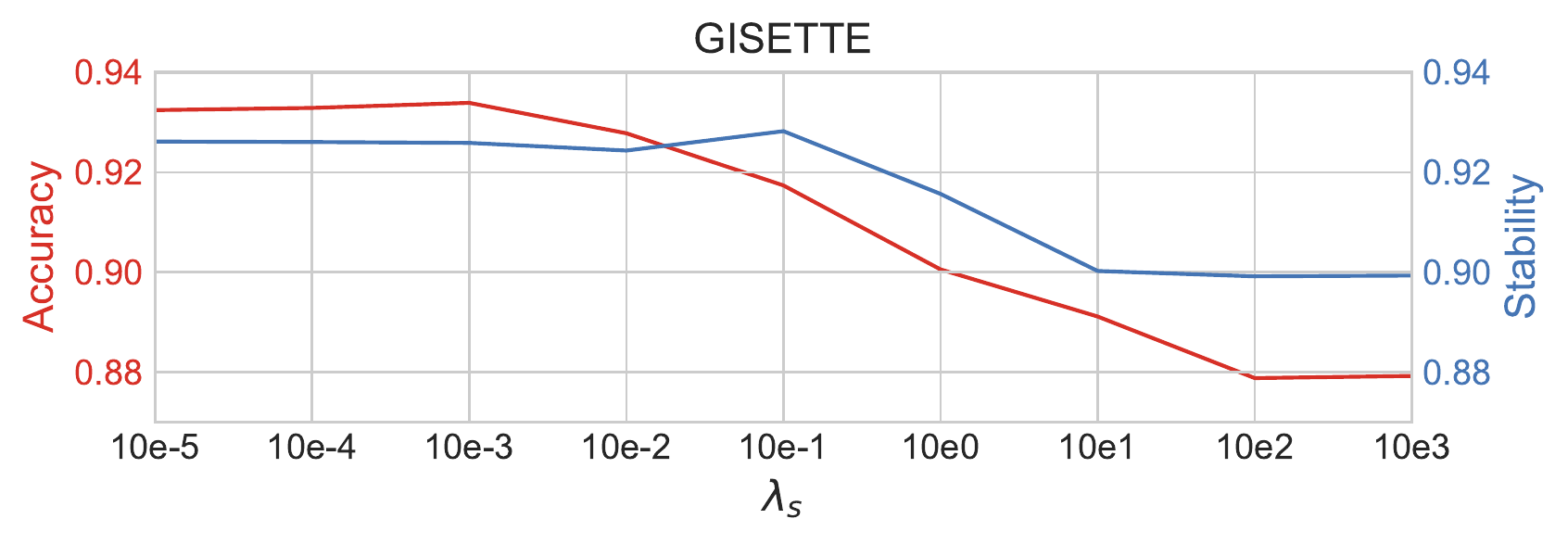}
    \caption{\emph{Choosing the $\lambda_s$ Penalty Term.} The larger $\lambda_s$, the higher we weigh uncertainty against importance when computing feature weights (see Eq. \eqref{eq:weight_objective}). Here, we illustrate the effect of $\lambda_s$ on the average accuracy per time step (red) and the average feature selection stability per time step (blue, according to Eq. \eqref{eq:stability}). We used the Gisette data set for illustration. If we increase $\lambda_s$, we initially also increase the stability. Yet, for $\lambda_s > 1$, the feature selection stability suffers. We observed similar effects for all remaining data sets. This suggests that stable feature sets depend not only on the uncertainty, but also on the importance of features. $\lambda_s=0.01$ was set as our default value throughout the experiments, because it provided the best balance between predictive power and stability in all data sets.}  
    \label{fig:lambda_effect}
\end{figure}
%------------ Lambda plot ----------%

\subsection{\fires~ Hyperparameters}
We assumed a standard normal distribution for all initial model parameters $\theta$, i.e. $\mu_k = 0, \sigma_k=1~\forall k$. The remaining hyperparameters of \fires\ were optimized in a grid search. Below we list the search space and final value of every hyperparameter:
\begin{itemize}
    \item[$\bm{\alpha}:$] Learning rate for updates of $\mu$ and $\sigma$, see Eq. \eqref{eq:gradient_ascent}: search space=$[0.01, 0.025, 0.1, 1, 10]$, final value=$0.01$.
    \item[$\bm{\lambda_s}:$] Penalty factor for uncertainty in the weight objective, see Eq. \eqref{eq:weight_objective}: search space=$[10e-5, 10e-4, 10e-3, 10e-2, 10e-1, 10e0, 10e1, 10e2, 10e3]$, final value=$0.01$. In Figure \ref{fig:lambda_effect}, we exemplary show the effect of different $\lambda_s$ for the Gisette data.
    \item[$\bm{\lambda_r}:$] Regularization factor in the weight objective, see Eq. \eqref{eq:weight_objective}: search space=$[0.01, 0.1, 1]$, final value=$0.01$.
\end{itemize}
There have been additional hyperparameters for the ANN and SDT based models:
\begin{itemize}
    \item \textbf{Learning rate (ANN+SDT):} Learning rate for gradient updates after backpropagation: search space=$[0.01,0.1,1]$, final value=$0.01$.
    \item \textbf{Monte Carlo samples (ANN+SDT):} Number of times we sample from the parameter distribution in order to compute the Monte Carlo approximation of the marginal likelihood: search space=$[3,5,7,9]$, final value=$5$.
    \item \textbf{\#Hidden layers (ANN):} Number of fully connected hidden layers in the ANN: search space=$[3,5,7]$, final value=$3$
    \item \textbf{Hidden layer size (ANN):} Nodes per hidden layer: search space=$[50,100,150,200]$, final value=$100$.
    \item \textbf{Tree depth (SDT):} Maximum depth of the SDT: search space=$[3,5,7]$, final value=$3$.
    \item \textbf{Penalty coefficient (SDT):} \citet{frosst2017distilling} specify a coefficient that is used to regularize the output of every inner node: search space=$[0.001,0.01,0.1]$, final value=$0.01$.
\end{itemize}
Note that we have evaluated every possible combination of hyperparameters, choosing the values that maximized the tradeoff between predictive power and stable feature sets. Similar to related work, we selected the search spaces empirically. The values listed above correspond to the default hyperparameters that we have used throughout the experiments.

\section{Data Sets and Preprocessing}
The Usenet data was obtained from \url{http://www.liaad.up.pt/kdus/products/datasets-for-concept-drift}. All remaining data sets are available at the UCI Machine Learning Repository \cite{Dua2019uci}. We used \emph{pandas.factorize()} to encode categorical features. Moreover, we normalized all features into a range $[0,1]$, using the \emph{MinMaxScaler()} of \emph{scipy}. Otherwise, we did not preprocess the data.

\section{Pseudo Code}
Algorithm \ref{alg:fires} depicts the pseudo code for the computation of feature weights at time step $t$. Note that the gradient of the log-likelihood might need to be approximated, depending on the underlying predictive model. In the main paper, we show how to use Monte Carlo approximation to compute the gradient for an Artificial Neural Net (ANN) and a Soft Decision Tree (SDT).

\begin{algorithm}
\caption{Feature weighting with \fires~ at time step $t$}
\label{alg:fires}
    \KwData{Observations $x_t \in \mathbb{R}^{B \times J}$ and corresponding labels $y_t = [y_{t1},..,y_{tB}]$ ($B$ = batch size, $J$ = no. of features); 
    Sufficient statistics of $K$ model parameters from the previous time step: $\mu_{t-1}, \sigma_{t-1} \in \mathbb{R}^K$}
    \KwResult{Feature weights: $\omega_t \in \mathbb{R}^J$; Updated statistics: $\mu_t, \sigma_t \in \mathbb{R}^K$}
    \Begin{
        \tcc{Define the log-likelihood $\mathcal{L}$ for some base model and compute the gradient}
        $\nabla_\mu \mathcal{L} \leftarrow$ Eq. \eqref{eq:gradient_likelihood}\;
        $\nabla_\sigma \mathcal{L} \leftarrow$ Eq. \eqref{eq:gradient_likelihood}\;
        \BlankLine
        \tcc{Update the sufficient statistics}
        $\mu_t = \mu_{t-1} + \alpha_\mu \nabla_\mu \mathcal{L}$\;
        $\sigma_t = \sigma_{t-1} + \alpha_\sigma \nabla_\sigma \mathcal{L}$\;
        \BlankLine
        \If{$\#parameters~K > \#input~features~J$}{
            $\mu'_t, \sigma'_t \in \mathbb{R}^J \leftarrow aggregate(\mu_t, \sigma_t)$\;
        }
        \BlankLine
        \tcc{Compute feature weights}
        $\omega_t \leftarrow$ Eq. \eqref{eq:omega}\;
    }
\end{algorithm}

\end{document}